\documentclass{article}
\usepackage[utf8]{inputenc}
\usepackage{amsmath}

\usepackage[colorlinks, citecolor=blue]{hyperref}

\usepackage{amssymb}

 \usepackage{algcompatible} %
\usepackage{algorithm}     %
 \usepackage{algorithmicx}  %
 \usepackage{algpseudocode}
\usepackage{lipsum}
\usepackage{algpseudocode}
\usepackage{caption}
\usepackage{xspace}
\usepackage{colortbl}
\usepackage{booktabs}

\usepackage[normalem]{ulem}

\usepackage{tikz}
\usetikzlibrary{calc,positioning, math, tikzmark}
\usepackage{pgfplots}

\usepackage{amsthm}

\usepackage[left=3cm,right=3cm,bottom=4cm]{geometry}
\usepackage{bbm}
\usepackage{bm}

\usepackage{comment}

\usepackage{wrapfig}
\newif\ifmylinenumbers
\mylinenumbersfalse

\ifmylinenumbers

\usepackage[mathlines]{lineno}
\linenumbers

\newcommand*\patchAmsMathEnvironmentForLineno[1]{%
   \expandafter\let\csname old#1\expandafter\endcsname\csname #1\endcsname
   \expandafter\let\csname oldend#1\expandafter\endcsname\csname end#1\endcsname
   \renewenvironment{#1}%
      {\linenomath\csname old#1\endcsname}%
      {\csname oldend#1\endcsname\endlinenomath}}%
\newcommand*\patchBothAmsMathEnvironmentsForLineno[1]{%
   \patchAmsMathEnvironmentForLineno{#1}%
   \patchAmsMathEnvironmentForLineno{#1*}}%
\AtBeginDocument{%
  \patchBothAmsMathEnvironmentsForLineno{equation}%
\patchBothAmsMathEnvironmentsForLineno{align}%
\patchBothAmsMathEnvironmentsForLineno{flalign}%
\patchBothAmsMathEnvironmentsForLineno{alignat}%
\patchBothAmsMathEnvironmentsForLineno{gather}%
\patchBothAmsMathEnvironmentsForLineno{multline}%
}

\newcounter{EditBlock}
\newcounter{CondEditBlock}

\else

\fi
\newcommand{\ev}[1]{\mathbb{E} \left [ #1 \right ] }
\newcommand{\evwrt}[2]{\mathbb{E}_{#1} \left [ #2 \right ] }

\newcommand{\snorm}[1]{\Vert #1 \Vert}

\newcommand{\one}[1]{\mathbbm{1} \left [ #1 \right ]}

\newcommand{\reals}{\mathbb{R}}

\newtheorem{lem}{Lemma}
\newtheorem*{lem*}{Lemma}
\newtheorem{thm}{Theorem}

\newtheorem{definition}{Definition}

\newtheorem{prop}{Proposition}

\newtheorem{example}{Example}

\newcommand{\tsum}{\mathsf{TSum}}

\newcommand{\vx}{\vec{x}}

\newcommand{\vz}{\vec{z}}

\newcommand{\vu}{\vec{u}}
\newcommand{\vup}{\vec{u}'}
\newcommand{\vv}{\vec{v}}
\newcommand{\vw}{\vec{w}}
\newcommand{\vmu}{\vec{\mu}}

\newcommand{\vxp}{\vec{x}'}

\let\vec\bm

\newcommand{\mN}{\mathcal{N}}

\newcommand{\mJ}{\mathcal{J}}
\newcommand{\mW}{\mathcal{W}}
\newcommand{\mB}{\mathcal{B}}
\newcommand{\mC}{\mathcal{C}}

\newcommand{\mD}{\mathcal{D}}

\DeclareMathOperator*{\argmin}{arg\,min}

\newcommand{\roberr}{\mathcal{L}}

\newcommand{\sgn}{\text{sgn}}

\algnewcommand\algorithmicinput{\textbf{Input:}}     %
\algnewcommand\INPUT{\item[\algorithmicinput]}       %
\algnewcommand\algorithmicoutput{\textbf{Output:}}   %
\algnewcommand\OUTPUT{\item[\algorithmicoutput]}     %
\algrenewcommand\algorithmicrequire{\textbf{Input:}} %
\algrenewcommand\algorithmicensure{\textbf{Output:}} %

\newcommand{\tvxp}{\widetilde{\vec{x}}'}

\newcommand{\mT}{\mathcal{T}}
\newcommand{\tmT}{\widetilde{\mathcal{T}}}
\newcommand{\tT}{\widetilde{T}}
\newcommand{\valpha}{\vec{\alpha}}
\newcommand{\vbeta}{\vec{\beta}}
\newcommand{\mF}{\mathcal{F}}

\newcommand{\admiss}{A} 
\newcommand{\tell}{\widetilde{\ell}}
\newcommand{\hmL}{\widehat{\mathcal{L}}}
\newcommand{\hvw}{\widehat{\vec{w}}}
\newcommand{\mZ}{\mathcal{Z}}
\newcommand{\tvx}{\widetilde{\vec{x}}}

\title{Generalization Properties of Adversarial Training for $\ell_0$-Bounded
  Adversarial Attacks\footnote{This paper was presented in part at 2023 IEEE Information Theory Workshop (ITW)}}

\author{Payam Delgosha\thanks{Department of Computer Science, University of
    Illinois at Urbana-Champaign, IL, \texttt{delgosha@illinois.edu}}
  \qquad Hamed Hassani\thanks{Department of Electrical and Systems Engineering,
    University of Pennsylvania, Philadelphia, PA, \texttt{hassani@seas.upenn.edu}}
  \qquad Ramtin Pedarsani\thanks{Department of Electrical and Computer
    Engineering, University of California, Santa Barbara, Santa Barbara, CA,
    \texttt{ramtin@ece.ucsb.edu}}}

\begin{document}

\maketitle

\begin{abstract}
    We have widely observed that neural networks are vulnerable to small additive perturbations to the input causing misclassification. In this paper, we focus on the $\ell_0$-bounded adversarial attacks, and aim to theoretically characterize the performance of adversarial training for an important class of \emph{truncated} classifiers. Such classifiers are shown to have strong performance empirically, as well as theoretically in the Gaussian mixture model, in the $\ell_0$-adversarial setting. The main contribution of this paper is to prove a novel generalization bound for the binary classification setting with $\ell_0$-bounded adversarial perturbation that is distribution-independent. Deriving a generalization bound in this setting has two main challenges: (i) the truncated inner
  product which is highly non-linear; and (ii) maximization over the $\ell_0$
  ball due to adversarial training is non-convex and highly  non-smooth. To
  tackle these challenges, we develop new
  coding techniques
  for bounding the combinatorial dimension of the truncated hypothesis class. 
\end{abstract}

\section{Introduction}
\label{sec:intro}

It is well-known that machine learning models are susceptible to adversarial attacks that can cause classification error. These attacks are typically in the form of a small norm-bounded perturbation to the input data that are carefully designed to incur misclassification  -- e.g. they can be an additive $\ell_p$-bounded perturbation for some $p\geq 0$ \cite{biggio2013evasion,szegedy,goodfellow2014explaining,carlini2017}.

There is an extensive body of prior work studying adversarial machine learning, most of which have focused on $\ell_2$ and $\ell_\infty$ attacks \cite{carlini2018,bhattacharjee2021consistent, bhattacharjee2020sample,wong2018provable,raghunathan2018certified}. 
To train models that are more robust against such attacks, adversarial training
is the state-of-the-art defense method. However, the success of the current
adversarial training methods is mainly based on empirical  evaluations \cite{Madry_ICLR}.  It is therefore imperative to study the fundamental limits of robust machine learning under different classification settings and attack models.

In this paper, we focus on the case of $\ell_0$-bounded attacks that
has been less investigated so far. In such attacks,  given an $\ell_0$ budget $k$, an adversary can
change $k$ entries of the input vector in an arbitrary fashion -- i.e. the adversarial perturbations belong to the $\ell_0$-ball of radius $k$. In contrast with
$\ell_p$-balls \textcolor{black}{($p \geq 1$)}, the  $\ell_0$-ball
is non-convex and non-smooth.
Moreover, the $\ell_0$-ball
contains inherent discrete (combinatorial) structures that can be exploited by
both the learner and the adversary. As a result, the $\ell_0$-adversarial
setting bears various challenges that are absent in common $\ell_p$-adversarial settings. In this regard, it has recently been shown that any piece-wise linear classifier, e.g. a
feed-forward neural  network with ReLu activations, completely fails in the
$\ell_0$ setting \cite{shamir2019simple}.

Perturbing only a few components of the data or signal has many real-world  applications
{\color{black}including} natural
language processing~\cite{jin2019bert},  malware
detection~\cite{grosse2016adversarial}, and physical attacks in object  detection~\cite{li2019adversarial}. 
There have been several prior works on $\ell_0$-adversarial attacks including white-box attacks that are gradient-based,
e.g.~\cite{carlini2017,papernot2016limitations,modas2019sparsefool}, and black-box attacks 
based on zeroth-order optimization,
e.g.~\cite{schott2018towards,croce2020sparse}. Defense strategies against
$\ell_0$-bounded attacks have also been proposed, e.g. defenses based on
randomized ablation~\cite{levine2020robustness} and defensive
distillation~\cite{papernot2016distillation}. None of the above works have
studied the fundamental limits of the $\ell_0$-adversarial setting
theoretically.



Recently, \cite{delgosha2021robust} proposed a 
classification algorithm called \texttt{FilTrun} and showed that it is robust
against $\ell_0$ adversarial attacks in a Gaussian mixture setting.
Specifically, they show that asymptotically as the data dimension gets large, no
other classification algorithm can do better than \texttt{FilTrun} in the
presence of adversarial attacks. Their algorithm consists of two component,
namely truncation and filteration. Although truncation can be efficiently
implemented, filteration is computationally expensive. Later,
\cite{beliaev2022efficient} proposed that  employing truncation in a neural
network architecture together with adversarial training results in a
classification algorithm which is robust against $\ell_0$ attacks. They proved
this for the  Gaussian mixture setting in an asymptotic scenario as the
data dimension goes to infinity. Furthermore, they demonstrated the
effectiveness of their proposed method against $\ell_0$ adversarial attacks  through experiments. 


In the previous theoretical results in $\ell_0$--bounded adversarial attacks,
it is assumed that the data distribution is  in the
form of a Gaussian mixture with known parameters, and the focus is on showing
the asymptotic optimality of the proposed architecture as the data dimension
goes to infinity. In practical supervised learning scenarios, we usually have
indirect access to the distribution through  i.i.d.\ training
data samples. In this setting, adversarial training is a natural method for
learning model parameters that are robust against adversarial attacks, as shown empirically in the prior work. 

Motivated by the theoretical and empirical success of truncation against
$\ell_0$ adversarial attacks,
in this paper we study its generalization
properties.
Generalization properties of adversarial training have been studied for other
adversarial settings   (for
instance \cite{schmidt2018adversarially}, \cite{yin2019rademacher}, \cite{montasser2019vc}, \cite{raghunathan2019adversarial}, \cite{attias2019improved}, \cite{feige2015learning}, \cite{awasthi2020rademacher}, \cite{khim2018adversarial}, \cite{najafi2019robustness}, and \cite{xing2021generalization}) mainly involving $\ell_p, p \geq 1$.
There are challenges inherent in the $\ell_0$ setting which make the standard techniques
 inapplicable. In this
paper, we discuss these challenges and  develop novel techniques to address this
problem. 
 We believe that the proposed mathematical techniques in this work are of independent interest and potentially 
 have applications in other generalization settings which are combinatorial
 in nature, such as neural network architectures equipped with  truncation
 components for robustness purposes.

\noindent\textbf{Summary of Contributions.} Our main contributions are as
follows:
\vspace{-2mm}
\begin{itemize}
\item
We consider a binary classification setting in the presence of an $\ell_0$
adversary with truncated linear classifiers as
our hypothesis class. We prove a generalization bound in this setting that is
distribution-independent, i.e.\ it holds for any distribution on the data (see
Theorem~\ref{thm:main-result} in Section~\ref{sec:main-results}).
\vspace{-2mm}
\item We observe that due
to the complex and combinatorial nature of our problem, the classical techniques
for bounding the combinatorial dimension  and the VC dimension are not
applicable to our setting (see the discussion in Section~\ref{sec:main-results}
and Appendix~\ref{app:tip-challenges}). To this end,  we introduce novel
techniques  that may 
be generalized to problems involving non-linear  and combinatorial operations.
\vspace{-2mm}
\item Specifically, there are two key challenges in
  bounding the combinatorial dimension in our setting: $(a)$ the truncated inner
  product which is highly non-linear, and $(b)$ the inner maximization over the $\ell_0$ 
  ball due to adversarial training, which
  is challenging to work with as it is non-convex and highly  non-smooth. It is worth mentioning that as \cite{montasser2019vc} has shown, it is possible that the original hypothesis class (truncated inner products in our case) has a finite VC dimension, but the corresponding adversarial setting  is only PAC learnable with an improper learning rule. Therefore, it is crucial in our work to resolve the two challenges individually and show that the VC dimension is finite even in the adversarial setting proving proper robust PAC learnability.
\vspace{-2mm}
\item We tackle the first challenge by employing a novel \emph{coding}
  technique, which encodes the sign of the truncated inner product  by a finite
  number of conventional inner products. This enables us to bound the
  growth function using the known bounds on the VC dimension of conventional
  inner product (see Proposition~\ref{prop:informal-tip-growth} in
  Section~\ref{sec:main-results} and its formal version
  Proposition~\ref{prop:tip-growth-bound}).
We tackle the second challenge by decomposing our loss function into two
  terms, one which does not involve maximization over the $\ell_0$, and one
  which involves studying the range of the truncated inner product over the
  $\ell_0$ ball (see the discussion in
  Section~\ref{sec:main-results-Pi-tT-bound-max} and
  Propositions~\ref{prop:informal-tT-max-growth} and~\ref{prop:max-l0-growth-bound}).
\end{itemize}







\section{Problem Formulation}
\label{sec:problem-formulation}

We consider the binary classification problem where the true label is denoted by
$y \in \{\pm 1\}$, and the feature vector has dimension $d$ and is denoted by
$\vx \in \reals^d$. We denote the joint distribution of $(\vx, y)$ by $\mD$. A
classifier is a function $\mC: \reals^d \rightarrow \{\pm 1\}$ which predicts
the label from the input. We consider the 0-1 loss $\ell(\mC; \vx, y) :=
\one{\mC(\vx) \neq y}$. 
We study classification under $\ell_0$ perturbations;
i.e.\ the adversary can perturb the input $\vx$ to $\vxp \in \reals^d$ where the
$\ell_0$ distance  between the two vectors which is  defined as
\begin{equation*}
  \snorm{\vx - \vxp}_0 := \sum_{i=1}^d \one{x_i \neq x'_i},
\end{equation*}
is bounded. In other words, the adversary can modify the input $\vx$ to any
other vector $\vxp$ within the $\ell_0$ ball of radius $k$ around $\vx$ defined as
\begin{equation*}
  \mB_0(\vx, k) := \{\vxp \in \reals^d: \snorm{\vx - \vxp}_0 \leq k\}.
\end{equation*}
Here, $k$ is the \emph{budget} of the adversary, and is effectively the number
of input coordinates that the adversary is allowed to change. The robust
classification error (or robust error for short) of a classifier $\mC$ when the
adversary has $\ell_0$ budget $k$ is defined as
\begin{equation}
  \label{eq:rob-err-def}
  \roberr_{\mD}(\mC, k) := \evwrt{(\vx, y) \sim \mD}{\tell_k(\mC; \vx, y)},
\end{equation}
where
\begin{equation}
  \label{eq:tell-def}
  \tell_k(\mC; \vx, y) := \max_{\vxp \in \mB_0(\vx, k)}\ell(\mC; \vxp, y).
\end{equation}
Here, $(\vx,y) \sim \mD$ means that the feature vector-label pair $(\vx, y)$ has
distribution $\mD$, and the maximum represents the adversary which can perturb
the input vector $\vx$ arbitrarily within the $\ell_0$ ball $\mB_0(\vx, k)$.


\vspace{2mm}

\noindent\textbf{Overview of Prior Results.} The authors of \cite{delgosha2021robust} study the above problem in the setting
of the Gaussian mixture model. More precisely, they assume that $y \sim
\text{Unif}\{\pm 1\}$, and conditioned on $y$, $\vx \sim \mN(y \vmu, \Sigma)$ is
normally distributed with mean $y \vmu$ and covariance matrix $\Sigma$. Here,
$\vmu \in \reals^d$, and $\Sigma$ is a positive-definite matrix. They study this
problem in an asymptotic fashion as the dimension $d$ goes to infinity. They
propose an algorithm called \texttt{FilTrun} and prove that it is
asymptotically optimal when $\Sigma$ is diagonal. Here, asymptotic optimality
means that asymptotically as the dimension $d$ goes to infinity, the robust error of \texttt{FilTrun} gets close to the
optimal robust  error, defined as the minimum of the
robust  error over all possible classifiers. \texttt{FilTrun} makes use of two
components, namely Filteration and Truncation.
\begin{itemize}
\item \textbf{Filteration} refers to a preprocessing phase, where upon receiving
  the perturbed data vector $\vxp$, we remove certain coordinates, or
  effectively set them to zero. The purpose of filteration is to remove the
  \emph{non-robust} coordinates. Let us denote the output of the filteration
  phase by $\tvxp$.
\item \textbf{Truncation} refers to applying the \emph{truncated inner product}
  of an appropriate weight vector $\vw$ by the output of the filteration phase
  $\tvxp$. More precisely,  the weight vector $\vw$ is chosen appropriately based on the
  distribution parameters $\vmu$ and $\Sigma$, and  the classification
  output is computed based as the sign of the truncated inner product $\langle \vw, \tvxp
  \rangle_k$ defined as follows. Let $\vu := \vw \odot \tvxp$ be the
  coordinate-wise product of $\vw$ and $\tvxp$, and let $u_{(1)} \leq u_{(2)}
  \leq \dots \leq u_{(d)}$ be the values in $\vu$ after sorting. Then,
  \begin{equation}
    \label{eq:truncated-inner-product-def}
    \langle \vw, \tvxp \rangle_k := \sum_{i=k+1}^{d-k} u_{(i)}.
  \end{equation}
  Effectively, $\langle \vw, \tvxp \rangle_k$ is the summation of the values in
  the coordinate-wise product of $\vw$ and $\tvxp$ after removing the $k$ largest
  and the $k$ smallest values. Note that when $k=0$, this reduces to the usual
  inner product, and removing the top and bottom $k$ values effectively removed
  the \emph{outliers} in the input, which are possibly caused by the adversary. 
\end{itemize}

Although truncation can be implemented in a computationally efficient way,
filteration turns out to be computationally expensive. The authors
in \cite{beliaev2022efficient}  have shown that in the Gaussian mixture setting
with diagonal covariance matrix, optimizing for the weight vector $\vw$ in the
class of truncated classifiers of the form $\mC^{(k)}_{\vw}(\vxp) := \sgn(\langle \vw, \vxp
\rangle_k)$ results in an asymptotically optimal classifier as the data
dimension goes to infinity. In other words, they  optimize for
$\vw$ in 
$\roberr_{\mD}(\mC^{(k)}_{\vw}, k)$ where $\mD$ is the Gaussian mixture
distribution. Note that the truncation parameter is chosen to be the same as the
adversary's budget $k$. 
Effectively, optimizing for $\vw$ automatically takes care of the filteration
phase since the weight of the ``non-robust'' features could be set to zero in $\vw$.
Motivated by this, they propose a neural network architecture
where the inner products in the first layer are replaced by truncated inner
products. Furthermore, they show through several experiments that in practice when we do not have
access to the distribution parameters, adversarial training as a proxy for
optimizing the model parameters results in an efficient and robust classifier.


\vspace{2mm}
\noindent\textbf{Adversarial Training for Parameter Tuning.}
In the theoretical analysis in above mentioned works,
it is assumed that the distribution $\mD$ is in the
form of a Gaussian mixture with known parameters, and therefore we can optimize
for the model parameters within the proposed architecture. In practice, we
usually do not have access to the distribution $\mD$. Instead, we usually have  i.i.d.\ training data samples $(\vx_1, y_1), \dots,
(\vx_n, y_n)$ distributed according to $\mD$. We stick to the usual
setting in the  adversarial attacks framework, where the training data is clean, while
the test data is perturbed by the adversary, and the objective is the robust
error at the test time. In this setting, adversarial training is a natural
choice for finding the model parameters. Motivated by the prior work
described above, we consider the
hypothesis class of truncated linear classifiers of the form
\begin{equation*}
  \mC^{(k)}_{\vw} : \vx \mapsto \sgn(\langle \vw, \vx \rangle_k),
\end{equation*}
parametrized by $\vw \in \reals^d$, where for $\alpha \in \reals$, $\sgn(\alpha)
:= +1$ when $\alpha \geq 0$, and $\sgn(\alpha) := -1$ when $\alpha < 0$. Also,
motivated by the prior work mentioned above, we set $k$ to be equal to the adversary's
budget. Furthermore, note that we are comparing $\langle \vw, \vx \rangle_k$
with zero. This is without loss of generality, since we may assume that there is a coordinate in $\vx$ with
constant value 1. Since we focus on this hypothesis class, we use the shorthand
notation $\roberr_{\mD}(\vw, k)$ for the robust error of the classifier
$\mC^{(k)}_{\vw}$, i.e.\ for $\vw \in \reals^d$, we define
\begin{equation}
  \label{eq:roberr-w-roberr-Ckw-shorthand}
  \roberr_{\mD}(\vw, k) := \roberr_{\mD}(\mC^{(k)}_{\vw}, k).
\end{equation}

Adversarial training in this scenario translates to choosing the hypothesis
parameter $\vw \in \reals^d$ by minimizing the adversarial empirical  loss
\begin{equation}
  \label{eq:hat-L-n-def}
  \hmL_n(\vw, k) := \frac{1}{n} \sum_{i=1}^n \tell_k(\mC^{(k)}_{\vw}; \vx_i, y_i).
\end{equation}
Recall that $\tell_k(\mC^{(k)}_{\vw}; \vx_i, y_i)$ which was defined
in~\eqref{eq:tell-def} is the maximum zero-one loss over the $\ell_0$ ball
around the $i$th data sample. Effectively, we assume that we have access to a
perfect  adversary during the training phase which allows us to have access to
$\tell_k(\mC^{(k)}_{\vw}; \vx_i, y_i)$. Let
\begin{equation}
  \label{eq:hwn-def}
  \hvw_n \in \argmin_{\vw \in \reals^d} \hmL_n(\vw,k),
\end{equation}
be the hypothesis parameter vector which is obtained by optimizing the above adversarial empirical
 loss over the training dataset. In this paper, we analyze the
generalization properties of the adversarial training for the above hypothesis
class of linear truncated classifiers. More precisely, our main question is
whether the robust error corresponding to  $\hvw_n$ (which is obtained by employing adversarial
training) converges to the robust error of the best truncated linear classifier
in our hypothesis class. More precisely, if
\begin{equation}
  \label{eq:wstar-def}
  \vw^\ast \in \argmin_{\vw \in \reals^d} \roberr_{\mD} (\mC^{(k)}_{\vw}, k),
\end{equation}
correspond to the best classifier in our hypothesis class, can we show that as the number of samples $n$ goes to
infinity,
$\roberr_{\mD}(\hvw_n, k)$ converges to
$\roberr_{\mD}(\vw^\ast, k)$? This question is formalized in the following definition.

\begin{definition}[robust PAC learnability]
\label{def:pac}
We say that a hypothesis class $\mathcal{H}$ is robust PAC learnable with respect to an $\ell_0$ adversary with budget $k$, if there exists a learning algorithm $\mathcal{A}$ such that for any $\epsilon, \delta >0$, and for any distribution $\mathcal{D}$, $\mathcal{A}$ maps i.i.d.\ data samples $\mathcal{S} = ((\vx_i, y_i), 1 \leq i \leq n)$ to $\mathcal{A}(\mathcal{S}) \in \mathcal{H}$, such that if $n > m(\epsilon, \delta)$, with probability at least $1 -\delta$, we have 
\begin{equation*}
\roberr_{\mD}(\mathcal{A}(\mathcal{S}),k) \leq \inf_{h \in \mathcal{H}} \roberr_{\mD}(h,k) + \epsilon.
\end{equation*}
\end{definition}

\vspace{2mm}
\noindent\textbf{Notation.}  $[n]$ denotes the set $\{1, \dots, n\}$. We denote vectors with boldface notation. Given a  vector $\vu= (u_1, \dots, u_d) \in \reals^d$, we denote by $u_{(1)} \leq 
\dots \leq  u_{(d)}$ the vector containing elements in $\vu$ in a non-decreasing
order. For instance, for $\vu = (3,1,1)$, we have $u_{(1)} = 1, u_{(2)} = 1$,
and $u_{(3)} = 3$. Given  $\vu, \vv \in \reals^d$, $\vu \odot \vw \in \reals^d$ is
defined to be the element-wise product of $\vu$ and $\vv$, i.e.\ its $i$th
coordinate is $u_i v_i$. For vectors $\vw, \vx \in \reals^d$ and integer $k <
d/2$,  with $\vu := \vw \odot \vx$, the $k$--truncated inner product between $\vw$ and $\vx$ is defined by 
$\sum_{i=k+1}^{d-k} u_{(i)}$.


\section{Main Results}
\label{sec:main-results}

Our main result is to show that for any distribution $\mD$, the robust error of
the classifier trained by adversarial training $\roberr_{\mD}(\hvw_n, k)$ converges to
$\roberr_{\mD}(\vw^\ast, k)$;  in other words, the class of truncated inner products in our adversarial setting is robustly PAC learnable as formalized in Definition~\ref{def:pac}. This is a direct consequence of the following
Theorem~\ref{thm:main-result}. The proof details are given in Section~\ref{sec:formal-analysis}. In the
remaining of this section, we explain the ideas and  main steps of
the proof.

\begin{thm}
  \label{thm:main-result}
  For any joint distribution $\mD$ on the label $y \in \{\pm 1\}$ and
  feature-vector $\vx \in \reals^d$, and any adversarial budget $0< k < d/2$,
  for $n > d+1$, if
  $\hvw_n$ denotes the model parameters obtained from adversarial training as
  in~\eqref{eq:hwn-def}, with probability at least $1- \delta$, we have
  \begin{equation*}
    \roberr_{\mD}(\hvw_n, k) \leq \roberr_{\mD}(\vw^\ast, k) + c \sqrt{\frac{d  \log \frac{en \left[ \binom{d}{2k} + \binom{d}{2}\right]}{d}}{ n}}  + 5 \sqrt{\frac{2 \log \frac{8}{\delta}}{n}},
  \end{equation*}
  where $c$ is a universal constant. 
\end{thm}

In order to obtain this bound, it suffices to  bound the Rademacher complexity of  our hypothesis class composed with the
loss $\tell$ defined in~\eqref{eq:tell-def} and use standard bounds such as
\cite[Theorem~26.5]{shalev2014understanding}.
Let $\tmT_{d,k} \subset \{-1,+1\}^{\reals^d \times \{\pm 1\}}$  be the class of
functions $\tT_{\vw, k}$ parametrized by $\vw \in \reals^d$ obtained by
applying the loss $\tell$ to the truncated linear classifier $\mC^{(k)}_{\vw}$, i.e.\
\begin{equation}
\label{eq:tT-wk-def}
  \tT_{\vw, k}(\vx, y) := \tell_k(\mC^{(k)}_{\vw}; \vx, y) = \max_{\vxp \in \mB_0(\vx, k)} \one{y \neq \sgn(\langle \vw, \vxp \rangle_k)}.
\end{equation}
Note that since the loss $\tell$ is the maximum of a zero-one loss, the
range of the functions in $\tmT_{d,k}$ is indeed $\{-1, +1\}$. 
To simplify the notation, with $\mZ:= \reals^d \times \{\pm 1\}$, we denote the
feature vector-label pair $(\vx,  y)$ by
$\vz \in \mZ$. In general,  
the Rademacher
complexity of a  function class $\mF \subset [-1,1]^{\mZ}$ is defined to be
\begin{equation}
  \label{eq:radem-complx-def}
  R_n(\mF) := \ev{\sup_{f \in \mF} \left| \frac{1}{n} \sum_{i=1}^n \epsilon_i f(\vz_i) \right|},
\end{equation}
where expectation is taken with respect to i.i.d.\ Rademacher random variables
$\epsilon_i \in \{\pm 1\}$ and i.i.d.\ samples $\vz_i = (\vx_i, y_i)$ with law $\mD$.
In the classification setting, where the function class is of the form $\mF
\subset \{\pm 1\}^{\mZ}$, 
the Rademacher complexity is related to a combinatorial dimension, known as the
VC dimension \cite{vapnik2015uniform}. Specifically,  the \emph{growth function}
of such a  function class
$\mF$ is defined as
\begin{equation*}
  \Pi_{\mF}(n) := \max\{|\{(f(z_1), \dots, f(z_n): f \in \mF\}|: z_i \in \mZ, 1 \leq i \leq n\}.
\end{equation*}
In words, the grown function captures the number of different $\pm 1$ patterns formed by
applying all the functions  in $\mF$ on $n$ data samples. Indeed,
$\Pi_{\mF}(n) \leq 2^n$, but in order to obtain useful bounds, we usually need
polynomial bounds on the growth function. We use the following standard bound on
the Rademacher complexity which follows from Massart lemma (see, for instance,
\cite[Lemma 26.8]{shalev2014understanding}) and holds  for any function class $\mF$ with range in $\{\pm 1\}$.
\begin{equation}
  \label{eq:rad-growth-bound}
  R_n(\mF) \leq \sqrt{\frac{2 \log (\Pi_{\mF}(n))}{n}}.
\end{equation}
Motivated by this, our strategy is to find  a bound for  
$\Pi_{\tmT_{d,k}}(n)$ which is polynomial in $n$.

Note that
from~\eqref{eq:tT-wk-def}, there are two challenges for bounding the
combinatorial dimension of the functions $\tT_{\vw, k}$: $(a)$ the truncated
inner product $\langle \vw, x \rangle_k$, and $(b)$ the maximization over the $\ell_0$
ball $\mB_0(\vx, k)$. These two components bring fundamental challenges beyond
those present in the usual machine learning scenarios where we deal with the
usual inner product and $\ell_p$ norms for $p \geq 1$. More precisely,
\vspace{-2mm}
\begin{enumerate}
\item  The truncated inner product $\langle \vw, \vx \rangle_k$ is
  not linear, i.e. $\langle \vw, \vx_1 + \vx_2\rangle_k$ is not necessarily
  equal to $\langle \vw, \vx_1 \rangle_k + \langle \vw, \vx_2 \rangle_k$.
As a concrete counterexample, let $d = 3$, $k=1$, $\vw = (1,1,1)$, $\vx_1 =
(10,9,-100)$, and $\vx_2 = (-100,1,2)$. Then, $\langle \vw, \vx_1 \rangle_k =
9$, $\langle \vw, \vx_2 \rangle_k = 1$, while $\langle \vw, \vx_1 + \vx_2
\rangle_k = -90$.
 Even
  worse than this, $\langle \vw, \vx_1 + \vx_2\rangle_k$ or even its sign cannot necessarily be
  uniquely determined by knowing both $\langle \vw, \vx_1 \rangle_k$ and $\langle
  \vw, \vx_2 \rangle_k$ or their signs (see Appendix~\ref{app:tip-challenges}
  for more examples).
\vspace{-2mm}
\item \label{item:ell-0-max-challenge}  The $\ell_0$ ball $\mB_0(\vx, k)$ is
  unbounded, non-convex,  and non-smooth. Due to
  this,  maximization over the
  ball is not tractable, unlike the case of $\ell_p$ balls for $p \geq 1$ (for
  instance \cite{yin2019rademacher} in the $\ell_\infty$ setting).
\end{enumerate}
\vspace{-2mm}
In Sections~\ref{sec:main-results-Pi-T-bound} and \ref{sec:main-results-Pi-tT-bound-max} below, we discuss the above two challenges. In order to
focus on these two challenges individually and to convey the main ideas, we first
study the function class corresponding to truncated inner products without
maximization over the $\ell_0$ ball. More precisely, let  $\mT_{d, k} \subset \{-1,1\}^{\reals^d}$ be the class of truncated inner
product functions of the form $T_{\vw,k} : \vx \mapsto \sgn(\langle \vw, \vx
\rangle_k)$, i.e.\ $\mT_{d,k} := \{T_{\vw, k}: \vw \in \reals^d\}$. In Section
\ref{sec:main-results-Pi-T-bound} below, we study the growth function $\Pi_{\mT_{d,k}}(n)$ of this function
class. Then, in Section \ref{sec:main-results-Pi-tT-bound-max}, we bring the maximization over the $\ell_0$ ball
into our discussion and study the growth function $\Pi_{\tmT_{d,k}}(n)$. Note
that in fact, $\mT_{d, k}$ is our hypothesis class, and $\tmT_{d,k}$ is the
composition of our hypothesis class with the maximized 0-1 loss $\tell$.

\subsection{Bounds on $\Pi_{\mT_{d,k}}(n)$}
\label{sec:main-results-Pi-T-bound}

Our main idea to bound the growth function $\Pi_{\mT_{d,k}}(n)$ is to encode the
truncated inner product in terms of a finite number of conventional inner products. Note that
$\langle  \vw, \vx \rangle_k$ is the sum of  $d-2k$ coordinates in $\vw
\odot \vx$. Therefore, if we know exactly which coordinates survive after
truncation, we can form the zero-one vector $\valpha$ where $\alpha_i$ is one
if the $i$th coordinate of $\vw \odot \vx$ survives after truncation, and is
zero otherwise. Then, it is easy to see that
\begin{equation*}
 \langle \vw, \vx \rangle_k =
\langle  \vw, \vx \odot \valpha \rangle,
\end{equation*}
 where the right hand side is the
conventional inner product (no truncation). However, the problem is that the vector
$\valpha$ is not  known beforehand, and it depends on the values in $\vw \odot
\vx$. But if we know the ordering of  $\vw \odot \vx$, we can form the
appropriate $\valpha$ by selecting the $d-2k$ intermediate values. In order to
address this, observe that the
ordering of values in $\vw \odot \vx$ can be determined by knowing the sign of
all $\binom{d}{2}$ pairwise terms of the form $w_i x_i - w_j x_j$ for $1 \leq i
< j \leq d$. But this can be in fact written as
\begin{equation*}
  w_i x_i - w_j x_j = \langle \vw, \vx \odot \vbeta \rangle,
\end{equation*}
where $\vbeta \in \reals^d$ is the vector whose $i$th coordinates is $+1$, 
$j$the coordinate is $-1$, and  other coordinates are zero. This discussion
motivates the following lemma
\begin{lem}[Lemma~\ref{lem:trp-sign-code} informal]
  \label{lem:informal-trp-sign-code}
  Given $\vw, \vx \in \reals^d$, $\sgn(\langle \vw, \vx \rangle_k)$ can be
  determined by knowing $\sgn(\langle \vw, \vx \odot \valpha_i \rangle)$ for $1
  \leq i \leq \binom{d}{2k}$, and $\sgn(\langle \vw, \vx \odot \vbeta_j
  \rangle)$ for $1 \leq j \leq \binom{d}{2}$. Here,  $\valpha_i$'s are the
  indicators of all the $\binom{d}{2k}$ subsets of size $d-2k$, and $\vbeta_j$'s are the vectors
  corresponding to all the $\binom{d}{2}$ pairs as in the above discussion. 
\end{lem}


\begin{figure}
  \centering
  \begin{minipage}[t]{0.4\linewidth}
    \centering
    \scalebox{0.6}{
  \begin{tabular}{@{}lccr@{}}
    \toprule
    $i$ & $\valpha_i$ & $\valpha_i \odot \vx$ & $\sgn(\langle \vw, \vx \odot \valpha_i \rangle)$ \\ \midrule
    $1$ & $(1,1,0,0)$ & $(1,-1,0,0)$  & {\color{red}$-1$} \\
    $2$ & $(1,0,1,0)$ & $(1,0,2,0)$ & {\color{red}$-1$}\\
    $3$ & $(1,0,0,1)$ & $(1,0,0,-3)$ & {\color{red}$-1$}\\
    $4$ & $(0,1,1,0)$ & $(0,-1,2,0)$ & {\color{blue}$+1$}\\
    $5$ & $(0,1,0,1)$ & $(0,-1,0,-3)$ & {\color{blue}$+1$}\\
    \rowcolor{orange!50}
    $6$ & $(0,0,1,1)$ & $(0,0,2,-3)$ & {\color{red}$-1$} \\ \bottomrule
  \end{tabular}
}%


\end{minipage}%
\begin{minipage}[t]{0.6\linewidth}
  \centering
  \scalebox{0.6}{
  \begin{tabular}{@{}lcccr@{}}
    \toprule
    $i$ & $\vbeta_i$ & $\vbeta_i \odot \vx$ & {$\sgn(\langle \vw, \vx \odot \vbeta_i \rangle)$} & conclusion \\ \midrule
    $1$ & $(1,-1,0,0)$ & $(1,1,0,0)$  & {\color{red}$-1$} & $w_1 x_1 < w_2 x_2$\\
    $2$ & $(1,0,-1,0)$ & $(1,0,-2,0)$ & {\color{red}$-1$} & $w_1 x_1 < w_3 x_3$ \\
    $3$ & $(1,0,0,-1)$ & $(1,0,0,3)$ & {\color{red}$-1$} & $w_1 x_1 < w_4 x_4$ \\
    $4$ & $(0,1,-1,0)$ & $(0,-1,-2,0)$ & {\color{blue}$+1$} & $w_2 x_2 \geq w_3 x_3$ \\
    $5$ & $(0,1,0,-1)$ & $(0,-1,0,3)$ & {\color{blue}$+1$} & $w_2 x_2 \geq w_4 x_4$ \\
    $6$ & $(0,0,1,-1)$ & $(0,0,2,3)$ & {\color{blue}$+1$} & $w_3 x_3 \geq w_4 x_4$ \\ \bottomrule
  \end{tabular}
}%

\end{minipage}
  \caption{Illustration of Lemma~\ref{lem:informal-trp-sign-code} for $d=4$,
    $k=1$, $\vx = (1,-1,2,-3)$, and $\vw = (-5,-4,-1,1)$. From $\sgn(\langle
    \vw, \vx \odot \vbeta_j \rangle)$ for $1 \leq j \leq 6$ on the right, we realize that $w_1 x_1
    \leq w_4 x_4 \leq w_3 x_3 \leq w_2 x_2$. This means that $\langle \vw, \vx
    \rangle_k = w_3 x_3 + w_4 x_4 = \langle \vw, \vx \odot \valpha_6 \rangle$
    whose sign can be read from the highlighted row on the left table.}
  \label{fig:coding}
\end{figure}


Figure~\ref{fig:coding} illustrates Lemma~\ref{lem:informal-trp-sign-code} through an example.
In fact, Lemma~\ref{lem:informal-trp-sign-code} suggests that $T_{w,k}(\vx) =
\sgn(\langle \vw, \vx \rangle_k)$ can be ``coded'' in terms of the signs of $\binom{d}{2k} +
\binom{d}{2}$  conventional inner products. Therefore, given $\vx_1, \dots, \vx_n \in
\reals^d$, and $\vw \in \reals^d$, we can form the following  $\pm 1$ matrix
with size $n \times (\binom{d}{2k} + \binom{d}{2})$
\begin{equation*}
  \begin{bmatrix}
    \sgn(\langle \vw, \vx_1 \odot \valpha_1 \rangle) & \dots &     \sgn(\langle \vw, \vx_1 \odot \valpha_{\binom{d}{2k}} \rangle) & \sgn(\langle \vw, \vx_1 \odot \vbeta_1 \rangle) & \dots & \sgn(\langle \vw, \vx_1 \odot \vbeta_{\binom{d}{2}} \rangle) \\
    \vdots & \vdots & \vdots & \vdots & \ddots & \vdots \\
    \sgn(\langle \vw, \vx_n \odot \valpha_1 \rangle) & \dots &     \sgn(\langle \vw, \vx_n \odot \valpha_{\binom{d}{2k}} \rangle) & \sgn(\langle \vw, \vx_n \odot \vbeta_1 \rangle) & \dots & \sgn(\langle \vw, \vx_n \odot \vbeta_{\binom{d}{2}} \rangle) \\
  \end{bmatrix}
\end{equation*}
Recall that the growth function $\Pi_{\mT_{d,k}}(n)$ is the maximum over
$(\vx_i: i \in [n])$ of  the number of possible configurations of $(\sgn(\langle \vw,
\vx_i \rangle): i \in [n])$ when $\vw$ ranges in $\reals^d$. Note that due to Lemma~\ref{lem:informal-trp-sign-code}, $(\sgn(\langle \vw,
\vx_i \rangle): i \in [n])$ is determined by knowing the above matrix, and
hence the growth function is bounded by the maximum over $(\vx_i:i \in [n])$ of
the total number of possible configurations of the above matrix when $\vw$
ranges in $\reals^d$. Since each entry in the above matrix is in the form of
a conventional inner
product, and the VC dimension of the conventional inner product is equal to $d$
(see, for instance \cite{shalev2014understanding}), the number of possible
configurations of the above matrix is bounded by a polynomial in the number of
entries in the matrix, which is proportional to $n$ (note that the number of
columns is $\binom{d}{2k} + \binom{d}{2}$ which does not scale with $n$).
This means,

\begin{prop}[Proposition~\ref{prop:tip-growth-bound} informal]
  \label{prop:informal-tip-growth}
  The growth function $\Pi_{\mT_{d,k}}(n)$ is bounded by a degree $d$ polynomial in $n$,
  whose  coefficients depend on $d$ and $k$.
\end{prop}

\subsection{Bounds on $\Pi_{\tmT_{d,k}}(n)$}
\label{sec:main-results-Pi-tT-bound-max}

Now, we extend the ideas from Section~\ref{sec:main-results-Pi-T-bound} to
bring the maximization over the $\ell_0$ ball into play and bound the growth
function of the function class $\tmT_{d,k}$. Observe that given a function
$\tT_{\vw,k}(.) \in \tmT_{d,k}$, we may write
\begin{equation}
  \label{eq:tT-max-1-I1-I2-1}
  \begin{aligned}
    \tT_{\vw, k}(\vx, y) &= \one{\exists \vxp \in \mB_0(\vx, k): y \neq \sgn(\langle \vw, \vxp \rangle_k)} \\
    &= \one{\sgn(\langle \vw, \vx \rangle_k) \neq y} \vee \one{\exists \vxp \in \mB_0(\vx, k): \sgn(\langle \vw, \vxp \rangle_k) \neq \sgn(\langle \vw, \vx \rangle_k)},
  \end{aligned}
\end{equation}
where $\vee$ denotes the logical OR. The first term is very similar to what we
discussed in Section~\ref{sec:main-results-Pi-T-bound}. Let us focus on the
second term, which we denote by $I_1(\vw, \vx)$. Equivalently, we may write
\begin{equation*}
  I_1(\vw, \vx) = \one{\sgn\left(\inf_{\vxp \in \mB_0(\vx, k)} \langle \vw, \vxp \rangle_k\right) \neq \sgn\left(\sup_{\vxp \in \mB_0(\vx, k)}\langle \vw, \vxp \rangle_k\right)},
\end{equation*}
where we let $\sgn(\infty) := +1$ and $\sgn(-\infty) := -1$. This motivates
studying the maximum and minimum values of the truncated inner product over the
$\ell_0$ ball. It is useful to define a
notation for this purpose. Given a vector $\vu \in \reals^d$, the truncated sum
of $\vu$ is defined as
\begin{equation}
  \label{eq:tsum-def}
  \tsum_k(\vu):= \sum_{i=k+1}^{d-k} u_{(i)}.
\end{equation}
Recall that $u_{(i)}$ denotes the $i$th smallest value in $\vu$. Observe that
$\langle \vw, \vx \rangle_k = \tsum_k(\vw \odot \vx)$. On the other hand, we
have
\begin{equation*}
  \{\vw \odot \vxp: \vxp \in \mB_0(\vx, k)\} \subset \mB_0(\vw \odot \vx, k).
\end{equation*}
Note that  if $\vw$ has some zero coordinates, the inclusion is strict, since
$\vw \odot \vxp$ is always zero in the zero coordinates  of $\vw$. However, if $w_i
\neq 0$ for all $i \in [d]$, the two sets are in fact equal. This means that if
$w_i \neq 0$ for all $i \in [d]$, with $\vu := \vw \odot \vx$, we have
\begin{equation}
\label{eq:I1-sgn-inf-sup}
  I_1(\vw, \vx) = \one{\sgn\left(\inf_{\vup \in \mB_0(\vu, k)} \tsum_k(\vup)\right) \neq \sgn\left(\sup_{\vup \in \mB_0(\vu, k)} \tsum_k(\vup)\right)}.
\end{equation}
It turns out that the maximum and minimum of the truncated sum can be
explicitly found, as is stated in the following Lemma~\ref{lem:tsum-l0-min-max},
whose proof is given in Appendix~\ref{app:lem-tsum-l0-min-max-proof}.

\begin{lem}
  \label{lem:tsum-l0-min-max}
  For $\vu \in \reals^d$, we have 
    \begin{align*}
      \min\{\tsum_k(\vup): \vup \in \mB_0(\vu, k)\} &= u_{(1)} + \dots + u_{(d-2k)} \\
      \max\{\tsum_k(\vup): \vup \in \mB_0(\vu, k)\} &= u_{(2k+1)} + \dots + u_{(d)}.
    \end{align*}    
\end{lem}

\definecolor{mygreen}{RGB}{62,188,150}

\begin{figure}
  \centering
  \scalebox{0.8}{
  \begin{tikzpicture}
    \begin{scope}[xshift=-3cm]
    \begin{scope}
      \fill[blue!30] (0.25,-0.5) rectangle (1.25,0.5);
      \fill[orange!50,rounded corners] (1.25,-0.5) rectangle (2.25,0.5);
      \fill[orange!50] (1.25,-0.5) rectangle (2,0.5);
      \draw[rounded corners, thick] (-2.25,-0.5) rectangle (2.25,0.5);
      \draw (0.25,-0.5) -- (0.25,0.5);
      \draw (1.25,-0.5) -- (1.25,0.5);
      \node[scale=0.7] at (-2.25,0.7) {$u_{(1)}$};
      \node[scale=0.7] at (0.25,0.7) {$u_{(d-2k)}$};
      \node[scale=0.7] at (1.25,0.7) {$u_{(d-k)}$};
      \node[scale=0.7] at (2.25,0.7) {$u_{(d)}$};
      \coordinate (lt) at (1.75,-0.5);
    \end{scope}

    \begin{scope}[yshift=-1.75cm]
      \fill[blue!30,rounded corners] (0.25,-0.5) rectangle (1.25,0.5);
      \fill[blue!30] (0.25,-0.5) rectangle (1,0.5);
      \draw[rounded corners, thick] (-2.25,-0.5) rectangle (1.25,0.5);
      \draw (0.25,-0.5) -- (0.25,0.5);
      \node[scale=0.7] at (-2.25,0.7) {$u_{(1)}$};
      \node[scale=0.7] at (0.25,0.7) {$u_{(d-2k)}$};
      \node[scale=0.7] at (1.25,0.7) {$u_{(d-k)}$};
      \fill[mygreen!60,rounded corners] (-3.25,-0.5) rectangle (-2.25,0.5);
      \draw[rounded corners, thick] (-3.25,-0.5) rectangle (-2.25,0.5);
      \coordinate (lb) at (-2.75,0.5);
    \end{scope}
    \draw[gray,thick,rounded corners,->] ($(lt)+(0,-0.05)$) -- ($(lt) + (0,-0.3)$) -- ($(lb)+(0,0.45)$) -- ($(lb)+(0,0.05)$);
    \node at (0,-2.7) {$(a)$};
    \end{scope}

    \begin{scope}[xshift=3cm]
    \begin{scope}
      \fill[blue!30] (-0.25,-0.5) rectangle (-1.25,0.5);
      \fill[orange!50,rounded corners] (-1.25,-0.5) rectangle (-2.25,0.5);
      \fill[orange!50] (-1.25,-0.5) rectangle (-2,0.5);
      \draw[rounded corners, thick] (-2.25,-0.5) rectangle (2.25,0.5);
      \draw (-0.25,-0.5) -- (-0.25,0.5);
      \draw (-1.25,-0.5) -- (-1.25,0.5);
      \node[scale=0.7] at (-2.25,0.7) {$u_{(1)}$};
      \node[scale=0.7] at (-0.25,0.7) {$u_{(2k)}$};
      \node[scale=0.7] at (-1.25,0.7) {$u_{(k)}$};
      \node[scale=0.7] at (2.25,0.7) {$u_{(d)}$};
      \coordinate (rt) at (-1.75,-0.5);
    \end{scope}

    \begin{scope}[yshift=-1.75cm]
      \fill[blue!30,rounded corners] (-0.25,-0.5) rectangle (-1.25,0.5);
      \fill[blue!30] (-0.25,-0.5) rectangle (-1,0.5);
      \draw[rounded corners, thick] (-1.25,-0.5) rectangle (2.25,0.5);
      \draw (-0.25,-0.5) -- (-0.25,0.5);
      \node[scale=0.7] at (-0.25,0.7) {$u_{(2k)}$};
      \node[scale=0.7] at (-1.25,0.7) {$u_{(k)}$};
      \node[scale=0.7] at (2.25,0.7) {$u_{(d)}$};
      \fill[mygreen!60,rounded corners] (3.25,-0.5) rectangle (2.25,0.5);
      \draw[rounded corners, thick] (3.25,-0.5) rectangle (2.25,0.5);
      \coordinate (rb) at (2.75,0.5);
    \end{scope}
    \draw[gray,thick,rounded corners,->] ($(rt)+(0,-0.05)$) -- ($(rt) + (0,-0.3)$) -- ($(rb)+(0,0.45)$) -- ($(rb)+(0,0.05)$);
    \node at (0,-2.7) {$(b)$};
    \end{scope}

  \end{tikzpicture}
  }
  \vspace{-2mm}
  \caption{$(a)$ Sorted elements in $\vu$ are illustrated on top, and $\vup \in
    \mB_0(\vu,k)$ on the bottom. To minimize $\tsum_k(\vup)$, we need to make the top
    $k$ elements in $\vu$ 
    (orange
    block) smaller than $u_{(1)}$ (green block). After truncating the
    green and blue blocks in $\vup$, we get $\tsum_k(\vup) =
    u_{(1)} + \dots + u_{(d-2k)}$. $(b)$ similarly, $u_{(2k+1)} +
    \dots + u_{(d)}$ is the maximum.}
  \label{fig:tsum-min-max}
\end{figure}
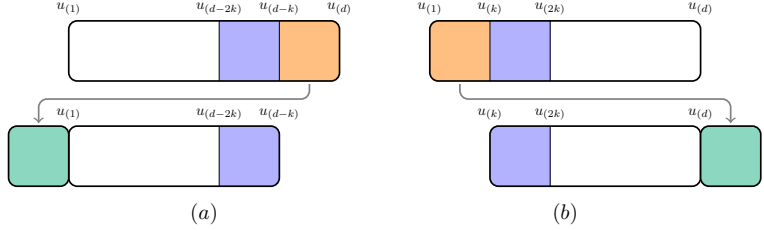


To have an intuitive understanding of this lemma,  note that to make
$\tsum_k(\vup)$ as small as possible, the best strategy is to make the top $k$
coordinates in $\vu$ arbitrarily small (say $-\infty$). After
truncation, this gives $u_{(1)} + \dots + u_{(d-2k)}$ (see Fig.~\ref{fig:tsum-min-max}). The
maximum term can be justified similarly.
Using Lemma~\ref{lem:tsum-l0-min-max} in~\eqref{eq:I1-sgn-inf-sup}, we
realize that if $w_i \neq 0$ for all $i \in [d]$, 
  \begin{equation*}
    I_1(\vw, \vx) =    \one{\sgn(u_{(1)} + \dots + u_{(d-2k)}) \neq \sgn(u_{(2k+1)} + \dots u_{(d)})},
  \end{equation*}
where $\vu := \vw \odot \vx$. Using this
in~\eqref{eq:tT-max-1-I1-I2-1},  if $w_i \neq 0$ for all $i \in
[d]$, we have
\begin{equation*}
  \tT_{\vw, k}(\vx, y) = \one{\sgn(\tsum_k(\vu)) \neq y} \vee \one{\sgn\left(\sum_{i=1}^{d-2k} u_{(i)}\right) \neq \sgn\left(\sum_{i=2k+1}^{d} u_{(i)}\right)},
\end{equation*}
where $\vu = \vw \odot \vx$. Note that all the three terms $\sgn(\tsum_k(\vu))$,
$\sgn(\sum_{i=1}^{d-2k} u_{(i)})$, and $\sgn(\sum_{i=2k+1}^{d} u_{(i)})$ that appear in the
above expression are in the form of the sign of the summation of some $d-2k$
coordinates in $\vu = \vw \odot \vx$ after sorting the elements in $\vu$.
Therefore, the same argument of Lemma~\ref{lem:informal-trp-sign-code} can be
used to show that all of these three terms are determined by knowing 
$\sgn(\langle \vw, \vx \odot \valpha_i \rangle)$ for $1
  \leq i \leq \binom{d}{2k}$, and $\sgn(\langle \vw, \vx \odot \vbeta_j
  \rangle)$ for $1 \leq j \leq \binom{d}{2}$.
Consequently, a similar coding technique as the one in
Section~\ref{sec:main-results-Pi-T-bound} can be used to show that the number of
configurations of $(\tT_{\vw,k}(\vx_i, y_i): i \in [n])$ as $\vw$ ranges over
all the vectors with nonzero coordinates is polynomial in $n$. The case where
$\vw$ contains some zero coordinates requires some technical work and is
addressed in Section~\ref{sec:formal-analysis}. To summarize, 

\begin{prop}[Proposition~\ref{prop:max-l0-growth-bound} informal]
  \label{prop:informal-tT-max-growth}
    The growth function $\Pi_{\tmT_{d,k}}(n)$ is bounded by a degree $d$ polynomial in $n$,
  whose  coefficients depend on $d$ and $k$.
\end{prop}

Using the bound of Proposition~\ref{prop:informal-tT-max-growth} (whose exact
form is given in Proposition~\ref{prop:max-l0-growth-bound})
in~\eqref{eq:rad-growth-bound}, we get an upper bound of order $\sqrt{\log n
  /n}$ for $R_n(\tmT_{d,k})$, which yields the bound of our
Theorem~\ref{thm:main-result}. The formal analysis of the proof steps are provided in Section~\ref{sec:formal-analysis}.


\section{Formal Analysis}
\label{sec:formal-analysis}

In this section, we cover the details of the proof steps in
Section~\ref{sec:main-results}. The proof of Theorem~\ref{thm:main-result} which is
given in Appendix~\ref{app:main-thm-proof}, follows directly 
from Proposition~\ref{prop:max-l0-growth-bound}, the bound
in~\eqref{eq:rad-growth-bound}, and \cite[Theorem~26.5]{shalev2014understanding}.

\subsection{A Growth Bound for   Truncated Inner
  Products}
\label{sec:formal-tip-bound}

In this section, we state and prove the formal version of
Proposition~\ref{prop:informal-tip-growth}, which is
Proposition~\ref{prop:tip-growth-bound}.
\begin{prop}
  \label{prop:tip-growth-bound} For $n
  > d+1$, we
  have
  \begin{equation}
    \label{eq:tip-growth-bound}
    \Pi_{\mT_{d,k}}(n) \leq \left( \frac{e n \left( \binom{d}{2k} + \binom{d}{2} \right)}{d} \right)^{d}.
  \end{equation}
\end{prop}

First  we need the formal version of Lemma~\ref{lem:informal-trp-sign-code},
which is Lemma~\ref{lem:trp-sign-code} below.
Let $\mJ_1, \dots, \mJ_{\binom{d}{2k}}$ be  the subsets of size $d-2k$ of $[d]$
in an arbitrary but fixed order, and define $\valpha_i= (\alpha_{i,1}, \dots, \alpha_{i,d}), 1 \leq i \leq
\binom{d}{2k}$  such that
$\alpha_{i,j}$ is 1 if $j \in \mJ_i$ and zero otherwise.
Also, define  $\vbeta_i: 1 \leq i \leq
\binom{d}{2}$ such that for each of them we pick two coordinates in $[d]$, set the smaller coordinate to 1, and the large coordinate to $-1$ (recall Fig.~\ref{fig:coding}).
The proof of Lemma~\ref{lem:trp-sign-code} is given in Appendix~\ref{app:trp-sign-code-proof}.

\begin{lem}
\label{lem:trp-sign-code}
    Assume that $\vw, \vx \in \reals^d$ are given, and
    let $\vu := \vw \odot \vx$. Then,
  for any subset $\mJ \subset [d]$ with $|\mJ| = d-2k$, the value of
  $\sgn(\sum_{j \in \mJ} u_{(j)})$ can be determined by knowing the set of signs $(\sgn(\langle \vw, \vx \odot \valpha_i \rangle):
  1 \leq i \leq \binom{d}{2k})$ and the set of signs $(\sgn(\langle \vw, \vx
  \odot \vbeta_j \rangle): 1 \leq j \leq \binom{d}{2})$.
\end{lem}

Here, we give the proof idea of Proposition~\ref{prop:tip-growth-bound}, and
refer the reader to Appendix~\ref{app:prop-trp-growth-proof} for a detailed proof.
  Fix $\vx_1, \dots, \vx_n \in \reals^d$. Let $m := \binom{d}{2k} +
  \binom{d}{2}$. For $1 \leq j \leq \binom{d}{2k}$, let $\vx_{i,j} := \vx_i
  \odot \valpha_j$, and for $1+\binom{d}{2k} \leq j \leq m$, let $\vx_{i,j} :=
  \vx_i \odot \vbeta_{j - \binom{d}{2k}}$.
  For $\vw \in \reals^d$, we define the $n \times m$ matrix $M(\vw)$ with $\pm
  1$ entries as follows:
  \begin{equation*}
    M(\vw) :=
    \begin{bmatrix}
      \sgn(\langle \vw, \vx_{1,1} \rangle)  & \dots & \sgn(\langle \vw, \vx_{1,m} \rangle) \\
      \vdots & \ddots & \vdots \\
      \sgn(\langle \vw, \vx_{n,1} \rangle) & \dots & \sgn(\langle \vw, \vx_{n,m} \rangle)
    \end{bmatrix}
  \end{equation*}
Note that from Lemma~\ref{lem:trp-sign-code}, $T_{\vw, k}(\vx_i) = \sgn(\langle
\vw, \vx_i \rangle_k)$ for $i\in[n]$ are determined by knowing $M(\vw)$. But since the VC dimension for conventional inner product in
dimension $d$ is $d$, and there are $nm$ entries in the above matrix, using the
Sauer-Shelah lemma, the number of possible configurations of $M(\vw)$ as $\vw$
ranges over $\reals^d$  is at
most $(enm/d)^d$, which is the desired bound.

\subsection{A Growth Bound for  the Function Class $\tmT_{d, k}$}
\label{sec:formal-growth-tmT-dk-bound}

In this section, we state and discuss the proof steps of  the formal version of
Proposition~\ref{prop:informal-tT-max-growth}, which is
Proposition~\ref{prop:max-l0-growth-bound} below. The proof of
Proposition~\ref{prop:max-l0-growth-bound} is given in Appendix~\ref{app:prop-max-growth-proof}.

\begin{prop}
  \label{prop:max-l0-growth-bound}
  For $n > d+1$, we have
  \begin{equation*}
    \Pi_{\tmT_{d,k}}(n) \leq 1 + \left( \frac{e n \left( \binom{d}{2k} + \binom{d}{2} \right)}{d} \right)^{d}.
  \end{equation*}
\end{prop}


Recall from the discussion in Section~\ref{sec:main-results-Pi-tT-bound-max}
that due to the decomposition in~\eqref{eq:tT-max-1-I1-I2-1}
and~\eqref{eq:I1-sgn-inf-sup}, we are interested in the minimum and maximum of
$\tsum_k(\vup)$ when $\vup \in \mB_0(\vu, k)$ with $\vu = \vw \odot \vx$. When
$w_i \neq 0$ for all $i \in [d]$, we saw that this is doable using
Lemma~\ref{lem:tsum-l0-min-max}. Lemma~\ref{lem:Aw-large-k-I1-I2} below suggests
that as long as $\vw$ has ``sufficiently'' many nonzero coordinates, the signs of
the minimum and maximum have a similar expression, even though their exact value
might be different. The proof of Lemma~\ref{lem:Aw-large-k-I1-I2} is given in Appendix~\ref{app:lem-aw-large-k-I1-I2-proof}.

\begin{lem}
  \label{lem:Aw-large-k-I1-I2}
  Given $\vw \in \reals^d$  such that $|\{i \in [d]: w_i \neq 0\}| > k$, 
  for all $\vx \in \reals^d$, we have 
  \begin{equation*}
    I_1(\vw, \vx) =    \one{\sgn(u_{(1)} + \dots + u_{(d-2k)}) \neq \sgn(u_{(2k+1)} + \dots u_{(d)})} =: I_2(\vw, \vx),
  \end{equation*}
  where $\vu = \vw \odot \vx$.
\end{lem}
Additionally, if $|\{i \in [d]: w_i \neq 0\}| \leq k$, $\langle \vw, \vxp \rangle_k$
is always zero. This  is stated in Lemma~\ref{lem:vw-at-most-k-nonzero-trp-zero}
below and is proved in Appendix~\ref{app:lem-at-most-k-nonzero-proof}.

\begin{lem}
\label{lem:vw-at-most-k-nonzero-trp-zero}
  Given $\vw \in \reals^d$ with $|\{i \in [d]: w_i \neq 0\}| \leq k$, for
  all $\vx \in \reals^d$ we have $\langle \vw,
  \vx \rangle_k = 0$.
\end{lem}

In fact, Lemmas~\ref{lem:Aw-large-k-I1-I2}
and~\ref{lem:vw-at-most-k-nonzero-trp-zero} together help us characterize the
two terms in~\eqref{eq:I1-sgn-inf-sup} depending on the number of nonzero
elements in $\vw$. Therefore, as was discussed in
Section~\ref{sec:main-results-Pi-tT-bound-max}, a coding technique similar to
that of Section~\ref{sec:formal-tip-bound} and the proof of Proposition~\ref{prop:tip-growth-bound} can be used to bound
$\Pi_{\tmT_{d,k}}(n)$ as in Proposition~\ref{prop:max-l0-growth-bound}. The proof of 
Proposition~\ref{prop:max-l0-growth-bound} is given in Appendix~\ref{app:prop-max-growth-proof}.



\section{Conclusion}
\label{sec:conclusion-colt}

In this paper, we proved a distribution-independent generalization bound for the binary classification
setting with $\ell_0$--bounded adversarial perturbation. We saw that deriving such generalization bound is
challenging, in particular due to  $(i)$ the nonlinearity of the truncated inner
product, and $(ii)$ non-smoothness and non-convexity of the $\ell_0$ ball. We
tackled these challenges by introducing a novel  technique which enables
us to bound the growth function of our hypothesis class.



\newcommand{\etalchar}[1]{$^{#1}$}

\appendix

\section{Challenges of Workign with the Truncated Inner Product}
\label{app:tip-challenges}

In this section, we provide some examples to show that the truncated inner
product is indeed challenging to work with, and does not behave similar to the
usual inner product.

\begin{example}
\textbf{Truncated inner product is not linear.} Let $d = 3$, $k=1$, $\vw = (1,1,1)$, $\vx_1 =
(10,9,-100)$, and $\vx_2 = (-100,1,2)$. Then, $\langle \vw, \vx_1 \rangle_k =
9$, $\langle \vw, \vx_2 \rangle_k = 1$, while $\langle \vw, \vx_1 + \vx_2
\rangle_k = -90$.
\end{example}

Even though the truncated inner product is not linear, there might be some hope
to write $\langle \vw, \vx_1+ \vx_2 \rangle_k$ in terms of (possibly a
complicated and non-linear) function of $\langle \vw, \vx_1 \rangle_k$ and
$\langle \vw, \vx_2 \rangle_k$. The following example shows that this is not the case.

\begin{example}
  \label{example:tip-sum-not-function}
\textbf{In general, $\langle \vw, \vx_1 + \vx_2 \rangle_k$ is not necessarily a
  function of $\langle \vw, \vx_1 \rangle_k$ and $\langle \vw, \vx_2
  \rangle_k$}. In order to show this, we introduce $\vw, \vx_1, \vx_2, \tvx_1$,
and $\tvx_2$ such that $\langle \vw, \vx_1 \rangle_k = \langle \vw, \tvx_1
\rangle$, $\langle \vw, \vx_2 \rangle_k = \langle \vw, \tvx_2 \rangle_k$, but
$\langle \vw, \vx_1 + \vx_2 \rangle_k \neq \langle \vw, \tvx_1 + \tvx_2
\rangle_k$. For instance, take $d = 3$, $k=1$, and
\begin{equation}
  \label{eq:tip-counterexample-2}
  \begin{aligned}
    \vw &= (1,1,1) \\
    \vx_1 &= (10,9,-100) \\
    \vx_2 &= (-100,1,2) \\
    \tvx_1 &= (0,9,10) \\
    \tvx_2 &= (0,1,10).
  \end{aligned}
\end{equation}
Then, we have
\begin{gather*}
  \langle \vw, \vx_1 \rangle_k = \langle \vx, \tvx_1 \rangle_k = 9 \\
  \langle \vw, \vx_2 \rangle_k = \langle \vx, \tvx_2 \rangle_k = 1,
\end{gather*}
but
\begin{equation*}
  \langle \vw, \vx_1 + \vx_2 \rangle_k = -90 \neq 10 = \langle \vw, \tvx_1 + \tvx_2 \rangle_k.
\end{equation*}
\end{example}

In our setting of classification where we want to bound the combinatorial
complexity, we usually do not need to know the exact value of the truncated
inner product, and knowing its sign might be enough. For the classical inner product,
the linearity clearly implies that for instance, if $\langle \vw, \vx_1 \rangle
> 0$ and $\langle \vw, \vx_2 \rangle > 0$, then $\langle \vw, \vx_1 + \vx_2
\rangle > 0$. Surprisingly, even this is not true for the truncated inner
product.

\begin{example}
  \textbf{$\langle \vw, \vx_1 \rangle_k > 0$ and $\langle \vw, \vx_2 \rangle >
    0$ do not necessarily imply that $\langle \vw, \vx_1 + \vx_2 \rangle_k >
    0$.} To see this, let $d = 3, k=1$, and take the vectors $\vw, \vx_1, \vx_2$
 from Example~\ref{example:tip-sum-not-function} above. Clearly, $\langle \vw,
 \vx_1 \rangle_k > 0$ and $\langle \vw, \vx_2 \rangle_k > 0$, but $\langle \vw,
 \vx_1 + \vx_2 \rangle_0 < 0$.
\end{example}

The above examples show that clearly the methods in classical machine learning
to bound the combinatorial complexity and the VC dimension cannot be extended to
the truncated inner product, and we really need new techniques such as the ones
presented in this paper.


\section{Proof of Lemma~\ref{lem:tsum-l0-min-max}}
\label{app:lem-tsum-l0-min-max-proof}

The following Lemma~\ref{lem:order-vector-equivalent} gives
an equivalent form for the quantities $u_{(i)}$, $1 \leq i \leq d$. The proof of
Lemma~\ref{lem:order-vector-equivalent}   is
straightforward and hence is omitted.

\begin{lem}
  \label{lem:order-vector-equivalent}
  Given $\vu = (u_1, \dots, u_d) \in \reals^d$, for $1 \leq i \leq d$, we have
  \begin{equation*}
    u_{(i)} = \min \{\alpha: |\{1 \leq j \leq d: u_j \leq \alpha\}| \geq i\}.
  \end{equation*}
\end{lem}

\begin{lem}
\label{lem:up-k+1-u-i}
  Assume that $\vu, \vup \in \reals^d$ are given such that $\snorm{\vu - \vup}_0
  \leq k$. Then for $1 \leq i \leq d -k$, we have
  \begin{equation*}
    u'_{(k+i)} \geq u_{(i)}.
  \end{equation*}
\end{lem}

\begin{proof}[Proof of Lemma~\ref{lem:up-k+1-u-i}]
 Fix $1 \leq i \leq d-k$. Using Lemma~\ref{lem:order-vector-equivalent} above,
 we have
 \begin{equation*}
   u'_{(k+i)} = \min \{\alpha: |\{1 \leq j \leq d: u'_j \leq \alpha\}| \geq k+i\}.
 \end{equation*}
 Let $\alpha^*$ be the minimum in the right hand side. Note that  there are at least $k+i$
 coordinates in $\vup$ with values no more than $\alpha^*$, let $1\leq j_1 <
 \dots <  j_{k+i} \leq d$ be the index of such coordinates. Since $\snorm{\vu -
   \vup}_0\leq k$, all the
 coordinates except for at most $k$ of them in $\vup$ are identical to $\vu$.
 Therefore, the value of $\vu$ and $\vup$ are different in at most $k$ of the
 indices $j_1, \dots, j_{k+1}$. Consequently, there are at least $i$
 coordinates in $j_1, \dots, j_{k+i}$ where $\vup$ and $\vu$ match. Thereby
 \begin{equation*}
   u_{(i)} = \min\{\beta: |\{1 \leq j \leq d: x_j \leq \beta\}| \geq i\} \leq \alpha^* = u'_{(k+i)},
 \end{equation*}
 which completes the proof.
\end{proof}

\begin{proof}[Proof of Lemma~\ref{lem:tsum-l0-min-max}]
  Using Lemma~\ref{lem:up-k+1-u-i}, we may write
  \begin{align*}
    \tsum_k(\vup) &= \sum_{i=k+1}^{d-k} u'_{(i)} \\
                  &= \sum_{i=1}^{d-2k} u'_{(k+i)} \\
    &\geq \sum_{i=1}^{d-2k} u_{(i)},
  \end{align*}
  which is the desired lower bound. On the other hand, since $\snorm{\vup -
    \vu}_0 = \snorm{\vu - \vup}_0 \leq k$, using Lemma~\ref{lem:up-k+1-u-i} after
  swapping the role of $\vu$ and $\vup$, we get
  \begin{align*}
    u_{(2k+1)} + \dots + u_{(d)} &= \sum_{i=k+1}^{d-k} u_{(k+i)} \\
                                 &\geq \sum_{i=k+1}^{d-k} u'_{(i)} \\
    &= \tsum_k(\vup),
  \end{align*}
  which is precisely the desired upper bound.

  Now let $u_{i_1} \leq \dots \leq u_{i_d}$ be an ordering of $\vu$ and let
  $\vup \in \mB_0(\vu, k)$ be such that
\begin{align*}
  &u'_{i_1} = u'_{i_2} = \dots =u'_{i_k} = u_{(d)} + 1 \\
  &u'_{i_j} = u_{i_j} \qquad \text{for } j > k.
\end{align*}
Clearly, the largest $k$ values of $\vup$ are $u_{(d)}+1$ which correspond to
coordinates $i_1, \dots, i_k$. Furthermore, the smallest $k$ values of $\vup$
are $u'_{i_{k+j}} = u_{i_{k+j}} = u_{(k+j)}$ for $1 \leq j \leq k$.
Consequently, after truncating the $k$ largest and the $k$ smallest values in
$\vup$, we get
\begin{equation*}
  \tsum_k(\vup) = u'_{i_{2k+1}} + \dots + u'_{i_d} = u_{i_{2k+1}} + \dots + u_{i_d} = u_{(2k+1)} + \dots + u_{(d)}.
\end{equation*}
Hence, $\vup$ attains the upper bound. Similarly, the vector $\vup \in
\mB_0(\vu, k)$ such that
 \begin{align*}
   &u'_{i_{d-k+1}} = u'_{i_{d-k+2}} = \dots =u'_{i_d} = u_{(1)} -1 \\
   &u'_{i_j} = u_{i_j} \qquad \text{for } j \leq d-k,
 \end{align*}
attains the lower bound. Hence, the inequalities are sharp and are in fact the
minimum and the maximum of $\tsum_k(\vup)$ as $\vup$ ranges over $\mB_0(\vu,
k)$. This completes the proof.
\end{proof}


\section{Proof of Lemma~\ref{lem:trp-sign-code}}
\label{app:trp-sign-code-proof}

\begin{proof}[Proof of Lemma~\ref{lem:trp-sign-code}]
Fix a subset $\mJ \subset [d]$ with $|\mJ| = d-2k$. Note that for $1 \leq j \leq
\binom{d}{2}$, if $\vbeta_j$ has value 1 at coordinate $a$ and $-1$ at coordinate $b$,
then we have
\begin{equation*}
  \langle \vw, \vx \odot \vbeta_j \rangle = w_a x_a - w_b x_b.
\end{equation*}
Therefore
\begin{equation*}
  \sgn(\langle \vw, \vx \odot \vbeta_j \rangle) =
  \begin{cases}
    1 & w_a x_a \geq w_b x_b  \\
    -1 & \text{otherwise}.
  \end{cases}
\end{equation*}
This means that given $\left(\sgn(\langle \vw, \vx
  \odot \vbeta_j \rangle): 1 \leq j \leq \binom{d}{2}\right)$, we can 
sort the elements in $\vu = \vw \odot \vx$ and obtain an ordering $i_1 \leq
\dots \leq i_d$ such that $u_{i_1} \leq u_{i_2} \leq \dots \leq u_{i_d}$. This
means that
\begin{equation*}
 \sum_{j \in \mJ} u_{(j)}  = \sum_{j \in \mJ} u_{i_j}.
\end{equation*}
But if $\tilde{\mJ} := \{i_j : j \in \mJ\}$, and $\valpha_{i(\tilde{\mJ})}$ is the indicator vector
associated to $\tilde{\mJ}$, we have
\begin{equation*}
  \sgn\left( \sum_{j \in \mJ} u_{(j)} \right) = \sgn\left( \sum_{j \in \tilde{\mJ}} u_j \right) = \sgn\left( \langle \vw, \vx \odot \valpha_{i(\tilde{\mJ})} \rangle  \right),
\end{equation*}
which is known given $\left(\sgn(\langle \vw, \vx \odot \valpha_i \rangle):
  1 \leq i \leq \binom{d}{2k}\right)$. The proof is complete since as was
discussed above, $\tilde{\mJ}$ is a function of $\left(\sgn(\langle \vw, \vx
  \odot \vbeta_j \rangle): 1 \leq j \leq \binom{d}{2}\right)$.
\end{proof}


\section{Proof of Proposition~\ref{prop:tip-growth-bound}}
\label{app:prop-trp-growth-proof}

\begin{proof}[Proof of Proposition~\ref{prop:tip-growth-bound}]
  Fix $\vx_1, \dots, \vx_n \in \reals^d$. Let $m := \binom{d}{2k} +
  \binom{d}{2}$. For $1 \leq j \leq \binom{d}{2k}$, let $\vx_{i,j} := \vx_i
  \odot \valpha_j$, and for $1+\binom{d}{2k} \leq j \leq m$, let $\vx_{i,j} :=
  \vx_i \odot \vbeta_{j - \binom{d}{2k}}$.
  For $\vw \in \reals^d$, we define the $n \times m$ matrix $M(\vw)$ with $\pm
  1$ entries as follows:
  \begin{equation*}
    M(\vw) :=
    \begin{bmatrix}
      \sgn(\langle \vw, \vx_{1,1} \rangle)  & \dots & \sgn(\langle \vw, \vx_{1,m} \rangle) \\
      \vdots & \ddots & \vdots \\
      \sgn(\langle \vw, \vx_{n,1} \rangle) & \dots & \sgn(\langle \vw, \vx_{n,m} \rangle)
    \end{bmatrix}
  \end{equation*}
Note that from Lemma~\ref{lem:trp-sign-code}, $T_{\vw, k}(\vx_i) = \sgn(\langle
\vw, \vx_i \rangle_k)$ for $1 \leq i \leq n$ are determines by knowing the
matrix $M(\vw)$. Therefore, the total number of configurations for $(T_{\vw,
  k}(\vx_1), \dots, T_{\vw, k}(\vx_n))$ when $T_{\vw, k}$ ranges over $\mT_{d,
  k}$ is bounded from above by the total number of configurations for the matrix
$M(\vw)$ when $\vw$ ranges over $\reals^d$. In other words,
\begin{equation}
  \label{eq:T-vx-1n-bound-M-w}
  |\{(T_{\vw, k}(\vx_1), \dots, T_{\vw, k}(\vx_n)): T_{\vw, k} \in \mT_{d,k}\}| \leq |\{M(\vw): \vw \in \reals^d\}|.
\end{equation}
Note that for $1 \leq i \leq n$ and $1 \leq j \leq m$, $\vx_{i,j}$ are vectors
in $\reals^d$.
Consequently, since the VC dimension of the sign of the inner product
in $\reals^{d}$ is $d$, and $n > d+1$, the Sauer–Shelah lemma (see, for instance,
\cite[Lemma 6.10]{shalev2014understanding})
implies that
\begin{equation}
  \label{eq:Mjw-conf-j-alpha}
  |\{M(\vw): \vw \in \reals^d\}| \leq \left( \frac{e nm}{d} \right)^{d}.
\end{equation}
Comparing this 
with~(\ref{eq:T-vx-1n-bound-M-w}), we realize that
\begin{equation*}
  |\{(T_{\vw, k}(\vx_1), \dots, T_{\vw, k}(\vx_n)): T_{\vw, k} \in \mT_{d,k}\}| \leq \left( \frac{e n\left( \binom{d}{2k} + \binom{d}{2} \right)}{d} \right)^{d}.
\end{equation*}
Since the same bound  holds for all $\vx_1, \dots, \vx_n$, we get the desired
bound.
\end{proof}


\section{Proof of Lemma~\ref{lem:Aw-large-k-I1-I2}}
\label{app:lem-aw-large-k-I1-I2-proof}

Given $\vw \in \reals^d$, we define
\begin{equation*}
  \admiss_{\vw} := \{i \in [d]: w_i \neq 0\},
\end{equation*}
and $\admiss_{\vw}^c := [d] \setminus \admiss_{\vw}$.
The following lemma will be crucial to prove Lemma~\ref{lem:Aw-large-k-I1-I2}.
The proof of Lemma~\ref{lem:inf-sup-w-x-k-equality-cond} below is given at the
end of this section.

\begin{lem}
  \label{lem:inf-sup-w-x-k-equality-cond}
  Assume that  $\vw, \vx \in \reals^d$ are given and define $\vu := \vw \odot \vx$.
  \begin{enumerate}
  \item We have 
    \begin{equation*}
      \inf\{\langle \vw, \vxp \rangle_k: \vxp \in \mB_0(\vx, k)\} \geq u_{(1)} + \dots +u_{(d-2k)}.
    \end{equation*}
    Moreover,  if $w_i \neq 0$ for all $i \in [d]$, or $|\{i: u_i \geq 0
    \text{ and } w_i \neq 0\}| \geq k$, then the infimum is minimum and equality
    holds. 
  \item We have 
    \begin{equation*}
      \sup\{\langle \vw, \vxp \rangle_k: \vxp \in \mB_0(\vx, k)\} \leq u_{(2k+1)} + \dots +u_{(d)}.
    \end{equation*}
    Moreover,  if $w_i \neq 0$ for all $i \in [d]$, or $|\{i: u_i \leq 0
    \text{ and } w_i \neq 0\}| \geq k$, then the supremum is maximum and
    equality holds.
  \end{enumerate}
\end{lem}

\begin{proof}[Proof of Lemma~\ref{lem:Aw-large-k-I1-I2}]
  Let $\vu = \vw \odot \vx$ and 
  \begin{equation*}
    u_{i_1} \leq u_{i_2} \leq \dots \leq u_{i_d},
  \end{equation*}
  be an ordering of the coordinates in $\vu$.  
 We divide the proof into cases, based on the sign
  of $u_{i_{k+1}} = u_{(k+1)}$ and $u_{i_{d-k}} = u_{(d-k)}$.

  \underline{Case 1: $u_{i_{k+1}} < 0$ and $u_{i_{d-k}} \leq 0$.} We have
  \begin{equation*}
    \langle \vw, \vx \rangle_k = \tsum_k(\vw \odot \vx) = u_{i_{k+1}} + \dots + u_{i_{d-k}} < 0.
  \end{equation*}
  This means that
  \begin{equation}
    \label{eq:case-1-I1-exists-positive}
    I_1(\vw, \vx)= \one{\exists \vxp \in \mB_0(\vx, k): \langle \vw, \vxp \rangle_k \geq 0}.
  \end{equation}
  On the other hand, note that $u_{i_{k+1}} < 0$ implies that $u_{i_1} , \dots,
  u_{i_k} < 0$. Since for all $i \in [d]$, we have $u_{i} = w_{i} x_{i}$, this means $w_{i_1}, \dots, w_{i_k}\neq 0$ and
  \begin{equation*}
    |\{i:u_i \leq 0, w_i \neq 0 \}| \geq |\{i_1, \dots, i_k\}| = k.
  \end{equation*}
  Therefore, Lemma~\ref{lem:inf-sup-w-x-k-equality-cond} implies that
  \begin{equation*}
    \max\{\langle \vw, \vxp \rangle_k: \vxp \in \mB_0(\vx, k)\}= u_{(2k+1)} + \dots + u_{(d)}.
  \end{equation*}
  Combining this with~\eqref{eq:case-1-I1-exists-positive}, we realize that
  \begin{equation}
    \label{eq:case-1-I1-one-2k+1-d-g-0}
  I_1(\vw, \vx) = \one{u_{(2k+1)} + \dots + u_{(d)} \geq 0}.
  \end{equation}
  Furthermore, since $u_{(d-2k)} \leq u_{(d-k)} = u_{i_{d-k}} \leq 0$, we have
  $u_{(1)}, \dots, u_{(d-2k)} \leq 0$. But $u_{(1)} \leq u_{(k+1)} = u_{i_{k+1}}
  < 0$. Thereby
  \begin{equation*}
    u_{(1)} + \dots + u_{(d-2k)} < 0.
  \end{equation*}
  Consequently,
  \begin{equation*}
    \one{u_{(2k+1)} + \dots + u_{(d)} \geq 0} = \one{\sgn(u_{(1)} + \dots + u_{(d-2k)}) \neq \sgn(u_{(2k+1)} + \dots + u_{(d)})} = I_2(\vw, \vx).
  \end{equation*}
  Combining this with~\eqref{eq:case-1-I1-one-2k+1-d-g-0}, we realize that
  $I_1(\vw, \vx) = I_2(\vw, \vx)$.

  \underline{Case 2: $u_{i_{k+1}} < 0$ and $u_{i_{d-k}} > 0$.} Note that for $1
  \leq j \leq k$, we have 
  $u_{i_j} \leq u_{i_{k+1}} < 0$. Also, since $u_i = w_i x_i$  for all $i\in
  [d]$, this means $w_{i_j} \neq 0$ for $1 \leq j \leq k$. Thereby
  \begin{equation*}
    |\{i:u_i \leq 0, w_i \neq 0\}| \geq |\{i_1, \dots, i_k\}| = k.
  \end{equation*}
  Consequently, Lemma~\ref{lem:inf-sup-w-x-k-equality-cond} implies that
  \begin{equation}
    \label{eq:case-2-max-u}
    \max\{\langle \vw, \vxp \rangle_k: \vxp \in \mB_0(\vx, k)\} = u_{(2k+1)} + \dots + u_{(d)}.
  \end{equation}
  Likewise, $u_{i_{d-k}} > 0$ implies that for $d-k+1 \leq j \leq d$, we
  have $u_{i_j} > 0$ and $w_{i_j} \neq 0$. Hence,
  \begin{equation*}
    |\{i: u_i \geq 0, w_i \neq 0\}| \geq |\{i_{d-k+1}, \dots, i_d\}| = k,
  \end{equation*}
  and from Lemma~\ref{lem:inf-sup-w-x-k-equality-cond} we have
  \begin{equation}
    \label{eq:case-2-min-u}
    \min\{\langle \vw, \vxp \rangle_k: \vxp \in \mB_0(\vx, k)\} = u_{(1)} + \dots + u_{(d-2k)}.
  \end{equation}
  Combining~\eqref{eq:case-2-max-u} and~\eqref{eq:case-2-min-u}, we realize that
  \begin{align*}
    I_1(\vw, \vx) &= \one{\exists \vxp \in \mB_0(\vx, k): \sgn(\langle \vw, \vxp \rangle_k) \neq \sgn(\langle \vw, \vx \rangle_k)} \\
                  &= \one{\sgn(    \max\{\langle \vw, \vxp \rangle_k: \vxp \in \mB_0(\vx, k)\}) \neq \sgn(    \min\{\langle \vw, \vxp \rangle_k: \vxp \in \mB_0(\vx, k)\})} \\
                  &= \one{\sgn(u_{(2k+1)} + \dots + u_{(d)}) \neq \sgn(u_{(1)} + \dots + u_{(d-2k)})} \\
    &= I_2(\vw, \vx).
  \end{align*}

  \underline{Case 3: $u_{i_{k+1}} = u_{i_{d-k}}= 0$.} Note that in this case, we
  have $u_{i_{k+1}} = \dots = u_{i_{d-k}} = 0$ and 
  \begin{equation*}
    \langle  \vw, \vx \rangle_k = u_{(k+1)} + \dots + u_{(d-k)} = u_{i_{k+1}} + \dots + u_{i_{d-k}} = 0.
  \end{equation*}
  Thereby,
  \begin{equation}
    \label{eq:case-3-I1-I3-simp}
    I_1(\vw, \vx) = \one{\exists \vxp \in \mB_0(\vx, k): \langle \vw, \vxp \rangle_k < 0} =: I_3(\vw, \vx).
  \end{equation}
  On the other hand, $u_{i_{k+1}} = 0$ implies $u_{i_j} \geq 0$ for $j \geq
  k+1$. Therefore
  \begin{equation*}
    u_{(2k+1)} + \dots + u_{(d)} \geq 0.
  \end{equation*}
  Hence, we have
  \begin{equation}
    \label{eq:case-3-I2-I4-simp}
    \begin{aligned}
      I_2(\vw, \vx) &= \one{\sgn(u_{(1)}+ \dots + u_{(d-2k)}) \neq \sgn(u_{(2k+1)} + \dots u_{(d)})} \\
      &= \one{\sgn(u_{(1)}+ \dots + u_{(d-2k)}) \neq 1} \\
      &= \one{u_{(1)}+ \dots + u_{(d-2k)} < 0} \\
      &=:I_4(\vw, \vx).
    \end{aligned}
  \end{equation}
  Combining~\eqref{eq:case-3-I1-I3-simp} and \eqref{eq:case-3-I2-I4-simp}, we
  realize that in order to show $I_1(\vw, \vx) = I_2(\vw, \vx)$, it suffices to
  show $I_3(\vw, \vx) = I_4(\vw, \vx)$. We verify this by considering 3 cases:
  \begin{itemize}
  \item   Assume that $I_4(\vw, \vx) = 0$.
    Then, using Lemma~\ref{lem:inf-sup-w-x-k-equality-cond}, we have
  \begin{equation*}
    \inf\{\langle \vw, \vxp \rangle_k: \vxp \in \mB_0(\vx, k)\} \geq u_{(1)} + \dots + u_{(d-2k)} \geq 0.
  \end{equation*}
  This means that for all $\vxp \in \mB_0(\vx, k)$, we have $\langle  \vw, \vxp
  \rangle_k \geq 0$ and $I_3(\vw, \vx) = 0 = I_4(\vw, \vx)$. 
\item Assume that $I_4(\vw, \vx) = 1$ and $|\{i: u_i \geq 0, w_i \neq 0\}| \geq
  k$. Then, from Lemma~\ref{lem:inf-sup-w-x-k-equality-cond}, we have
  \begin{equation*}
    \min\{\langle \vw, \vxp \rangle_k: \vxp \in \mB_0(\vx, k)\} = u_{(1)} + \dots +u_{(d-2k)}.  
  \end{equation*}
  Consequently, we have
  \begin{align*}
    I_3(\vw, \vx) &= \one{\exists \vxp \in \mB_0(\vx, k): \langle \vw, \vxp \rangle_k < 0} \\
                  &= \one{    \min\{\langle \vw, \vxp \rangle_k: \vxp \in \mB_0(\vx, k)\} < 0} \\
                  &= \one{u_{(1)} + \dots +u_{(d-2k)} < 0} \\
                  &= I_4(\vw, \vx).
  \end{align*}
\item Assume that $I_4(\vw, \vx) = 1$ and $|\{i: u_i \geq 0, w_i \neq 0\}| <
  k$. Note that by assumption of the lemma, we have $|\{i \in [d]: w_i \neq 0\}|
  > k$. We may write
  \begin{equation*}
    |\admiss_{\vw}| = |\{i \in [d]: w_i \neq 0\}| = |\admiss_{\vw}^+| + |\admiss_{\vw}^-|,
  \end{equation*}
  where
  \begin{align*}
    \admiss_{\vw}^+ &:= \{i \in [d]: u_i \geq 0, w_i \neq 0\} \\
    \admiss_{\vw}^- &:= \{i \in [d]: u_i < 0, w_i \neq 0\}.
  \end{align*}
  Since $|\admiss_{\vw}^+| < k$ and $|\admiss_{\vw}^+| + |\admiss_{\vw}^-| > k$,
  we have $|\admiss_{\vw}^-| > 0$.
  Now, we construct $\vxp \in
  \mB_0(\vx, k)$ as
  \begin{equation*}
    x'_i =
    \begin{cases}
      -w_i & i \in \admiss_{\vw}^+ \\
      x_i & \text{otherwise}.
    \end{cases}
  \end{equation*}
  Note that $|\admiss_{\vw}^+| < k$ guarantees that $\vxp \in \mB_0(\vx, k)$.
  Let $\vup := \vw \odot \vxp$ and note that for $i \in \admiss_{\vw}^+$, $u'_i
  = w_i x'_i = - w_i^2$. But $w_i \neq 0$ for $i \in \admiss_{\vw}^+$. Therefore, $u'_i <
  0$ for $i \in \admiss_{\vw}^+$. On the other hand, for $i \in
  \admiss_{\vw}^-$, we have $u'_i = u_i < 0$. Furthermore, for $i \in [d]
  \setminus \admiss_{\vw}$, we have $u'_i = u_i = w_i x_i = 0$. To sum up, we
  showed that for $i \in \admiss_{\vw} = \admiss_{\vw}^+ \cup \admiss_{\vw}^-$,
  we have $u'_i < 0$, and for $i \notin \admiss_{\vw}$, we have $u'_i = 0$. Note
  that $|\admiss_{\vw}| > k$. This means that there at least $k+1$ negative
  coordinates in $\vup$, and the remaining of the coordinates are zero.
  Therefore, at least one negative value  will survive after truncating the
  bottom $k$ coordinates in $\vup$, plus perhaps some other negative and some
  other zero values. Hence,
  \begin{equation*}
    \langle \vw, \vxp \rangle_k = \tsum_k(\vup) < 0.
  \end{equation*}
  Consequently, we have
  \begin{equation*}
    I_3(\vw, \vx) = \one{\exists \vxp \in \mB_0(\vx, k): \langle \vw, \vxp \rangle_k < 0} = 1 = I_4(\vw, \vx).
  \end{equation*}
\end{itemize}
To sum up, we showed that $I_3(\vw, \vx) = I_4(\vw, \vx)$ in all the cases, and
hence $I_1(\vw, \vx) = I_2(\vw, \vx)$ as was discussed above.

\underline{Case 4: $u_{i_{k+1}} \geq 0$ and $u_{i_{d-k}} > 0$.} We have
\begin{equation*}
  \langle \vw, \vx \rangle_k = \tsum_k(\vw \odot \vx) = u_{i_{k+1}} + \dots + u_{i_{d-k}} > 0.
\end{equation*}
This means that
\begin{equation}
  \label{eq:case-4-I1-exists-negative}
  I_1(\vw, \vx) = \one{\exists \vxp \in \mB_0(\vx, k): \langle \vw, \vx \rangle_k < 0}.
\end{equation}
On the other hand, since $u_{i_{d-k}} > 0$, for $j \geq d-k$, we have $u_{i_j}
\geq u_{i_{d-k}} > 0$. Since $u_i = w_i x_i$, this means that for $d-k \leq j
\leq d$, we have $u_{i_j} \geq 0$ and $w_{i_j} \neq 0$. Thereby
\begin{equation*}
  |\{i\in [d]: u_i \geq 0, w_i \neq 0\}| \geq |\{i_{d-k+1}, \dots, i_d\}| = k.
\end{equation*}
Consequently, Lemma~\ref{lem:inf-sup-w-x-k-equality-cond} implies that
\begin{equation*}
  \min\{\langle \vw, \vxp \rangle_k: \vxp \in \mB_0(\vx, k)\} = u_{(1)} + \dots + u_{(d-2k)}.
\end{equation*}
Combining this with~\eqref{eq:case-4-I1-exists-negative}, we have
\begin{equation}
  \label{eq:case-4-I1-one-1-d-2k-negative}
  I_1(\vw, \vx) = \one{u_{(1)} + \dots u_{(d-2k)} < 0}.
\end{equation}
Furthermore, for $j \geq 2k+1$, we have $u_{i_j} \geq u_{i_{k+1}} \geq 0$. This
means that $u_{(2k+1)} + \dots + u_{(d)} \geq 0$. Using this
in~\eqref{eq:case-4-I1-one-1-d-2k-negative} above, we realize that
\begin{equation*}
  I_1(\vw, \vx) = \one{\sgn(u_{(1)} + \dots + u_{(d-2k)}) \neq \sgn(u_{(2k+1)} + \dots + u_{(d)})} = I_2(\vw, \vx).
\end{equation*}
Consequently, we verified that $I_1(\vw, \vx) = I_2(\vw, \vx)$ in all the four
cases. This completes the proof.
\end{proof}

\begin{proof}[Proof of Lemma~\ref{lem:inf-sup-w-x-k-equality-cond}]
  We give the proof of the lower bound (part 1). The proof of the upper bound
  (part 2) is similar and hence is omitted.
  Note that
  \begin{equation*}
    \{\vw \odot \vxp: \vxp \in \mB_0(\vx, k)\} \subset \mB_0(\vu, k). 
  \end{equation*}
  Therefore, Lemma~\ref{lem:tsum-l0-min-max} implies that
  \begin{equation}
    \label{eq:inf-wxp-u1-proof-1}
  \begin{aligned}
    \inf\{\langle \vw, \vxp \rangle_k: \vxp \in \mB_0(\vx, k)\} &\geq  \min \{\tsum_k(\vup): \vup \in \mB_0(\vu, k)\} \\
    &= u_{(1)} + \dots + u_{(d-2k)}.
  \end{aligned}
  \end{equation}
  On the other hand, if $w_i \neq 0$ for all $i \in [d]$, or equivalently if
  $|\admiss_{\vw}^c| = 0$, we have $\{\vw \odot
  \vxp: \vxp \in \mB_0(\vx, k)\} \subset \mB_0(\vu, k)$ and
  indeed~\eqref{eq:inf-wxp-u1-proof-1} becomes inequality. Now we show that if
  $|\admiss_{\vw}^c| \neq 0$ and $|\{i: u_i \geq 0 \text{ and } w_i \neq 0\}|
  \geq k$, then the inequality also becomes equality. In order to show this, let
  \begin{equation*}
    u_{i_1} \leq u_{i_2} \leq \dots \leq u_{i_d},
  \end{equation*}
  be an ordering of the coordinates in $\vu$ in ascending order. Note that since
  $\{i: u_i \geq 0, w_i \neq 0\}| \geq k$, we can choose  the  ordering so
  that the coordinates corresponding to the top $k$ elements fall in
  $\admiss_{\vw}$, i.e.
  \begin{equation}
    \label{eq:i-d-k+1-id-in-admiss}
    \{i_{d-k+1}, \dots, i_d\} \subseteq \admiss_{\vw}.
  \end{equation}
  Now, define $\vxp \in \mB_0(\vx, k)$ as follows. For $d - k +1 \leq j \leq d$,
  let $x'_{i_j}$ be such that $w_{i_j} x'_{i_j} = u_{(1)} - 1$. Note that this
  is doable since due to~\eqref{eq:i-d-k+1-id-in-admiss}, we have $w_{i_j} \neq
  0$ for $d - k +1 \leq j \leq d$. On the other hand, for $j \leq d -k$, let
  $x'_{i_j} = x_{i_j}$. Note that with this specific perturbation, the $k$
  minimum values in $\vw \odot \vxp$ are
  precisely the coordinates $i_{d-k+1}, \dots, i_d$. Therefore
  \begin{equation*}
    \langle \vw, \vxp \rangle_k = \tsum_k(\vw \odot \vxp) = u_{(1)} + \dots u_{(d-2k)},
  \end{equation*}
  and hence~\eqref{eq:inf-wxp-u1-proof-1} becomes equality.
\end{proof}


\section{Proof of Lemma~\ref{lem:vw-at-most-k-nonzero-trp-zero}}
\label{app:lem-at-most-k-nonzero-proof}

\begin{proof}[Proof of Lemma~\ref{lem:vw-at-most-k-nonzero-trp-zero}]
  The assumption $|\{i \in [d]: w_i \neq 0\}| \leq k$ means
  that all but at most $k$ elements in $\vw$ are zero. This implies that all but
  at most $k$ elements of $\vw \odot \vx$ are  zero, and there are at least
  $d-k$ zero elements in $\vw \odot \vx$. Therefore, since there are no more
  than $k$  nonzero elements, after truncation, they will be removed and only
  zero elements will survive. This means $\langle \vw, \vx \rangle_k =
  \tsum_k(\vw \odot \vx) = 0$.
\end{proof}


\section{Proof of Proposition~\ref{prop:max-l0-growth-bound}}
\label{app:prop-max-growth-proof}

\begin{proof}[Proof of Proposition~\ref{prop:max-l0-growth-bound}]
Fix $\vx_1, \dots, \vx_n \in \reals^d$ and $y_1, \dots, y_n \in \{\pm 1\}$. From
Lemma~\ref{lem:vw-at-most-k-nonzero-trp-zero}, if $|\{i \in [d]: w_i \neq 0\}|
\leq k$, we have $\langle \vw, \vx'_i \rangle_k = 0$ for all $\vx'_i \in
\mB_0(\vx, k)$ and $1 \leq i \leq n$. This means that $\tT_{\vx, k}(\vx_i, y_i)
= \one{y_i = -1}$. Thereby,
\begin{equation}
  \label{eq:conf-count-tT-1+-bound}
  |\{(\tT_{\vw, k}(\vx_1, y_1), \dots, \tT_{\vw, k}(\vx_n, y_n)): \tT_{\vw, k} \in \tmT_{d,k}\}| \leq 1 +   |\{(\tT_{\vw, k}(\vx_1, y_1), \dots, \tT_{\vw, k}(\vx_n, y_n)): \vw \in \mW_{>k}\}|, 
\end{equation}
where
\begin{equation*}
  \mW_{> k} := \{\vw \in \reals^d: |\{i \in [d]: w_i \neq 0\}| > k\}.
\end{equation*}
But from~\eqref{eq:tT-max-1-I1-I2-1} and Lemma~\ref{lem:Aw-large-k-I1-I2}, for
$\vw \in \mW_{>k}$ and $1 \leq i \leq n$, we have
\begin{equation*}
  \tT_{\vw, k}(\vx_i, y_i) = \one{\sgn(\langle \vw, \vx_i \rangle_k) \neq y_i} \vee \one{\sgn(u^{(i)}_{(1)} + \dots + u^{(i)}_{(d-2k)}) \neq \sgn(u^{(i)}_{(2k+1)} + \dots u^{(i)}_{(d)})},
\end{equation*}
where $\vu^{(i)} := \vw \odot \vx_i$. Note that all the sign terms in the above
expression are of the form $\sgn(\sum_{j \in \mJ}(u^{(i)}_j))$ for some $\mJ
\subset [d]$ with $|\mJ| = d-2k$. Specifically, with $\mJ_1 := \{k+1, \dots,
d-k\}$, $\mJ_2:= \{1, \dots, d-2k\}$, and $\mJ_3 := \{2k+1, \dots, d\}$, we have
\begin{align*}
  \sgn(\langle \vw, \vx_i \rangle_k) &= \sgn\left( \sum_{j \in \mJ_1} u^{(i)}_j  \right) \\
  \sgn(u^{(i)}_{(1)} + \dots + u^{(i)}_{(d-2k)}) &= \sgn\left( \sum_{j \in \mJ_2}u^{(i)}_j \right)\\
  \sgn(u^{(i)}_{(2k+1)} + \dots u^{(i)}_{(d)}) &= \sgn\left( \sum_{j \in \mJ_3}u^{(i)}_j \right).
\end{align*}
Thereby, Lemma~\ref{lem:trp-sign-code} implies that all of these sign terms are
functions of $(\sgn(\langle \vw, \vx_{i,j} \rangle): 1 \leq j \leq m)$ where $m
= \binom{d}{2k} + \binom{d}{2}$, for $1 \leq j \leq \binom{d}{2k}$, $\vx_{i,j} =
\vx_i \odot \valpha_j$, and for $1 + \binom{d}{2k} \leq j \leq m$, $\vx_{i,j} =
\vx_i \odot \vbeta_{j - \binom{d}{2k}}$. Therefore, we may employ a proof
strategy  similar to that of Proposition~\ref{prop:tip-growth-bound}. More
precisely, for $\vw \in \mW_{>k}$, we define the $n \times m$ matrix $M(\vw)$ with $\pm
  1$ entries as follows:
  \begin{equation*}
    M(\vw) :=
    \begin{bmatrix}
      \sgn(\langle \vw, \vx_{1,1} \rangle)  & \dots & \sgn(\langle \vw, \vx_{1,m} \rangle) \\
      \vdots & \ddots & \vdots \\
      \sgn(\langle \vw, \vx_{n,1} \rangle) & \dots & \sgn(\langle \vw, \vx_{n,m} \rangle)
    \end{bmatrix}
  \end{equation*}
  The above discussion implies that
  \begin{equation}
    \label{eq:tT-conf-count-mW-big-k-Mw-bound}
    |\{(\tT_{\vw, k}(\vx_1, y_1), \dots, \tT_{\vw, k}(\vx_n, y_n)): \vw \in \mW_{>k}\}| \leq |\{M(\vw): \vw \in \mW_{> k}\}| \leq |\{M(\vw): \vw \in \reals^d\}|.
  \end{equation}
  But as we saw in the proof of Proposition~\ref{prop:tip-growth-bound}, since
  $n > d+1$, we have
  \begin{equation*}
    |\{M(\vw): \vw \in \reals^d\}| \leq \left( \frac{e n \left( \binom{d}{2k} + \binom{d}{2} \right)}{d} \right)^{d}.
  \end{equation*}
  Comparing this with~\eqref{eq:tT-conf-count-mW-big-k-Mw-bound}
  and~\eqref{eq:conf-count-tT-1+-bound}, we realize that
  \begin{equation*}
    |\{(\tT_{\vw, k}(\vx_1, y_1), \dots, \tT_{\vw, k}(\vx_n, y_n)): \tT_{\vw, k} \in \tmT_{d,k}\}| \leq 1 + \left( \frac{e n \left( \binom{d}{2k} + \binom{d}{2} \right)}{d} \right)^{d}.
  \end{equation*}
  Since this bound holds for all $\vx_1, \dots, \vx_n$ and $y_1, \dots, y_n$, we
  realize that
  \begin{align*}
    \Pi_{\tmT_{d,k}}(n) &= \max\{|\{(\tT_{\vw, k}(\vx_1, y_1), \dots, \tT_{\vw, k}(\vx_n, y_n)): \tT_{\vw, k} \in \tmT_{d,k}\}|: \vx_1, \dots, \vx_n \in \reals^d, y_1, \dots, y_n \in \{\pm 1\}\} \\
    &\leq 1 + \left( \frac{e n \left( \binom{d}{2k} + \binom{d}{2} \right)}{d} \right)^{d}.
  \end{align*}
  This completes the proof.
\end{proof}


\section{Proof of Theorem~\ref{thm:main-result}}
\label{app:main-thm-proof}

Using \cite[Theorem 26.5]{shalev2014understanding}, since our loss function
$\tell$ defined in~\eqref{eq:tell-def} is bounded by $1$, with probability at least
$1-\delta$, we have
\begin{equation}
  \label{eq:gen-bound-rad-1}
  \roberr_{\mD}(\hvw_n, k) \leq \roberr_{\mD}(\vw^\ast, k) + 2R_n(\tmT_{d,k})  + 5 \sqrt{\frac{2 \log \frac{8}{\delta}}{n}}.
\end{equation}
On the other hand,  using Massart lemma (see, for instance,
\cite[Lemma 26.8]{shalev2014understanding}, )we have 
\begin{equation}
  \label{eq:rad-growth-bound-in-proof}
  R_n(\mF) \leq \sqrt{\frac{2 \log ( \Pi_{\mF}(n))}{n}}.
\end{equation}
Using Proposition~\ref{prop:max-l0-growth-bound}, for $n > d+1$, we have
\begin{equation*}
  \Pi_{\tmT_{d,k}}(n) \leq 1 + \left( \frac{e n \left( \binom{d}{2k} + \binom{d}{2} \right)}{d} \right)^{d} \leq 2 \left( \frac{e n \left( \binom{d}{2k} + \binom{d}{2} \right)}{d} \right)^{d}.
\end{equation*}
Therefore,
\begin{align*}
  2 \log (2 \Pi_{\tmT_{d,k}}(n)) &\leq 2 \log 2 + 2d \log \frac{en \left[ \binom{d}{2k} + \binom{d}{2} \right]}{d} \\
  &\leq (2 + 2 \log 2)d  \log \frac{en \left[  \binom{d}{2k} + \binom{d}{2}  \right]}{d}.
\end{align*}
Using this in~\eqref{eq:rad-growth-bound-in-proof} and substituting
into~\eqref{eq:gen-bound-rad-1}, we realize that the desired bound in
Theorem~\ref{thm:main-result} holds with $c = 2\sqrt{2 + 2 \log 2}$. This
completes the proof.

\end{document}